\documentclass[lettersize,journal]{IEEEtran}

\usepackage{amsmath,amsfonts,amssymb}
\usepackage{dsfont}
\usepackage{amsthm}
\usepackage{algorithmic}
\usepackage{algorithm}
\usepackage{array}
\usepackage{booktabs}
\usepackage[caption=false,font=normalsize,labelfont=sf,textfont=sf]{subfig}
\usepackage{textcomp}
\usepackage{stfloats}
\usepackage{url}
\usepackage{graphicx}
\usepackage{pgfplots}
\usepackage{cite}
\usepackage[hidelinks]{hyperref}
\usepackage[capitalise]{cleveref}
\crefname{equation}{Eq.}{Eqs.}
\Crefname{equation}{Equation}{Equations}
\crefname{section}{Sec.}{Secs.}
\Crefname{section}{Section}{Sections}
\crefname{figure}{Fig.}{Figs.}
\Crefname{figure}{Fig.}{Figs.}
\crefname{table}{Table}{Tables}
\Crefname{table}{Table}{Tables}
\crefname{algorithm}{Alg.}{Algs.}
\Crefname{algorithm}{Algorithm}{Algorithms}

\pgfplotsset{compat=1.18}
\usetikzlibrary{positioning,arrows.meta,decorations.pathreplacing}

\newtheorem{proposition}{Proposition}

\newtheorem{remark}{Remark}

\newcommand{\Iinf}{I_t}
\newcommand{\Iref}{\widehat{I}^{\,\mathrm{ref}}_t}
\newcommand{\Amat}{\mathbf{A}_t}
\newcommand{\tr}{{\mathsf{T}}}

\begin{document}

\title{O-SCD++: Replay-Free Online Scene Change Detection via Recursive Evidence Fusion}

\author{Chamuditha~Jayanga~Galappaththige,
        Jason~Lai,
        Lloyd~Windrim,
        Timothy Patten,
        Donald~G.~Dansereau,
        Niko~Suenderhauf,
        and~Dimity~Miller%
\thanks{C.~J.~Galappaththige, N.~Suenderhauf and D.~Miller are with the QUT Centre for Robotics, Queensland University of Technology, Australia, and with the Australian Research Council Research Hub in Intelligent Robotic Systems for Real-Time Asset Management (ARIAM). E-mail: \{ahangama, d24.miller\}@qut.edu.au.}%
\thanks{J.~Lai and D.~G.~Dansereau are with the Australian Centre for Robotics, The University of Sydney, Australia. D.~G.~Dansereau is also with ARIAM.}%
\thanks{L.~Windrim and T.~Patten are with Abyss Solutions and with ARIAM.}%
}%

\markboth{}{}

\maketitle

\begin{abstract}
Online scene change detection requires an autonomous agent to detect what has changed in a previously mapped environment while it is still on site, observing from unconstrained viewpoints and without access to future observations. This causal setting is strictly harder than offline detection, and online methods have long trailed offline accuracy. Our earlier method, O\nobreakdash-SCD, closed that gap at over 10 frames per second, but its multi-view fusion is an iterative optimization that refits a change field over a 3D Gaussian splatting map of the reference scene to all past observations at every frame. This step consumes over half of the per-frame budget and forces accuracy to be traded against speed. We present O\nobreakdash-SCD++, a replay-free formulation that eliminates this trade-off. Because the reference map stays fixed, rendering the change field is a linear operation. Each new observation can therefore be carried onto the map's Gaussians by a single backward pass of the renderer and fused by a two-scalar recursion per Gaussian. Past observations are never stored or revisited. O\nobreakdash-SCD++ also makes the change/no-change decision without labels or iterative optimization. It sets each frame's threshold to the exact minimizer of O\nobreakdash-SCD's self-supervised objective over threshold masks, using one sort, and the distance to this threshold gives a per-pixel decision margin that predicts errors. On PASLCD, O\nobreakdash-SCD++ improves change detection accuracy (mIoU) from $0.489$ to $0.575$ and throughput from $11.8$ to $25.6$ frames per second, exceeding O\nobreakdash-SCD's offline result while remaining causal. Code will be released upon acceptance.
\end{abstract}


\section{Introduction}
\label{sec:intro}

\IEEEPARstart{A}{n} autonomous agent that revisits a previously mapped environment -- an inspection robot returning to an electrical substation or a construction site, a service robot re\nobreakdash-entering a room it mapped the day before -- must work out what has changed since its last visit, and must do so while suppressing nuisance variation such as shadows, reflections, or illumination shifts. The answer is only actionable if it arrives \emph{while the agent is still on site}: a change report delivered after the deployment has ended cannot trigger a closer inspection, a re-planned route, or an intervention.

Multi-view scene change detection (SCD) has accordingly developed along two regimes. In the \emph{offline} regime, complete pre- and post-change captures are available and change is inferred with full hindsight. Work in this regime has matured rapidly, moving from image pairs at near-identical viewpoints~\cite{sakurada_change_2015,alcantarilla2018street,daudt2018fully,cyws2d} to pose-agnostic, label-free formulations that reason inside a photorealistic 3D representation of the reference scene~\cite{galappaththige2025multi,lu20253dgs,galappaththige2026pixels,jiang2025gaussian,zhou2026d}. In the \emph{online} regime, the agent must emit a decision for every frame as it arrives, from a partial and incrementally growing stream, with no access to the future. This is the regime that matters for acting on site, and it is strictly harder than the offline counterpart.

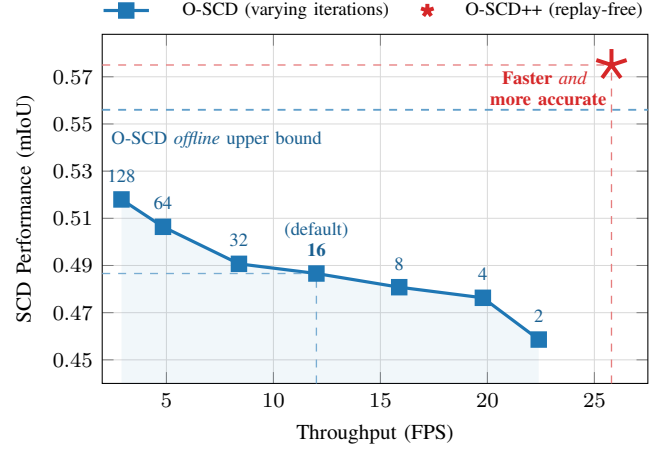
\begin{figure}[t]
  \centering
  \definecolor{oscdblue}{RGB}{31,119,180}
\definecolor{oursred}{RGB}{214,39,40}
\definecolor{shadegray}{RGB}{230,230,230}

\begin{tikzpicture}

\begin{axis}[
    width=8.8cm,
    height=6.2cm,
    xlabel= {Throughput (FPS)},
    ylabel={SCD Performance (mIoU)},
    xmin=2, xmax=27.5,
    ymin=0.44, ymax=0.588,
    xtick={5,10,15,20,25},
    ytick={0.45, 0.47, 0.49, 0.51, 0.53, 0.55, 0.57},
    grid=major,
    grid style={line width=0.3pt, draw=gray!30},
    major grid style={line width=0.3pt, draw=gray!30},
    tick label style={font=\footnotesize},
    label style={font=\footnotesize},
    legend style={
        font=\fontsize{7}{8}\selectfont,
        at={(0.5,1.02)},
        anchor=south,
        legend columns=-1,
        draw=none,
        fill=none,
        inner sep=1pt,
        column sep=10pt,
    },
    legend image post style={
        mark size=2.5pt,
    },
    clip=false,
]

\addplot[
    fill=oscdblue,
    fill opacity=0.06,
    draw=none,
    forget plot,
] coordinates {
    (2.9, 0.44)
    (2.9, 0.518)
    (4.84, 0.5064)
    (8.39, 0.4907)
    (11.8, 0.489)
    (15.88, 0.4808)
    (19.79, 0.4763)
    (22.40, 0.4586)
    (22.40, 0.44)
} -- cycle;

\addplot[
    color=oscdblue,
    mark=square*,
    mark size=2.5pt,
    line width=1.2pt,
    mark options={fill=oscdblue},
] coordinates {
    (2.9,   0.518)
    (4.84,  0.5064)
    (8.39,  0.4907)
    (11.8, 0.489)
    (15.88, 0.4808)
    (19.79, 0.4763)
    (22.40, 0.4586)
};
\addlegendentry{O-SCD (varying iterations)}

\addplot[
    dashed,
    line width=0.7pt,
    color=oscdblue!70,
    forget plot,
] coordinates {
    (2, 0.556)
    (27.5, 0.556)
};
\node[font=\scriptsize, color=oscdblue!80!black, anchor=south west]
    at (axis cs:2.0, 0.536) {O-SCD \emph{offline}};

\addplot[
    only marks,
    mark=star,
    mark size=6pt,
    line width=1.5pt,
    color=oursred,
    mark options={fill=oursred, draw=oursred},
] coordinates {
    (25.6, 0.575)
};
\addlegendentry{O-SCD++ (replay-free)}

\node[font=\scriptsize, color=oscdblue!80!black, anchor=south, yshift=3pt]
    at (axis cs:2.9, 0.518) {128};
\node[font=\scriptsize, color=oscdblue!80!black, anchor=south, yshift=3pt]
    at (axis cs:4.84, 0.5064) {64};
\node[font=\scriptsize, color=oscdblue!80!black, anchor=south, yshift=3pt]
    at (axis cs:8.39, 0.4907) {32};
\node[font=\scriptsize, color=oscdblue!80!black, anchor=south, yshift=3pt, align=center]
    at (axis cs:11.8, 0.489) {(default)\\\textbf{16}};
\node[font=\scriptsize, color=oscdblue!80!black, anchor=south, yshift=3pt]
    at (axis cs:15.88, 0.4808) {8};
\node[font=\scriptsize, color=oscdblue!80!black, anchor=south, yshift=3pt]
    at (axis cs:19.79, 0.4763) {4};
\node[font=\scriptsize, color=oscdblue!80!black, anchor=south, yshift=3pt]
    at (axis cs:22.40, 0.4586) {2};

\addplot[
    dashed,
    line width=0.5pt,
    color=oscdblue!60,
    forget plot,
] coordinates {
    (2, 0.489) (11.8, 0.489)
};
\addplot[
    dashed,
    line width=0.5pt,
    color=oscdblue!60,
    forget plot,
] coordinates {
    (11.8, 0.44) (11.8, 0.489)
};

\addplot[
    dashed,
    line width=0.5pt,
    color=oursred!60,
    forget plot,
] coordinates {
    (2, 0.575) (25.6, 0.575)
};
\addplot[
    dashed,
    line width=0.5pt,
    color=oursred!60,
    forget plot,
] coordinates {
    (25.6, 0.44) (25.6, 0.575)
};

\node[font=\scriptsize, color=oursred, anchor=east, align=center]
    at (axis cs:26, 0.565) {\textbf{Faster} \textit{and}\\\textbf{ more accurate}};

\end{axis}
\end{tikzpicture}
  \caption{The state-of-the-art O\protect\nobreakdash-SCD~\cite{galappaththige2025changes} fuses multi-view evidence by iterative optimization over replayed cue maps, which forces detection quality to be traded against speed (blue, annotated with the per-frame iteration budget). Our novel replay-free formulation (red star) derived in \cref{sec:method} eliminates this trade-off: it is simultaneously more accurate and more than twice as fast. Our formulation surpasses even O\protect\nobreakdash-SCD's offline performance while operating strictly causally.}
  \label{fig:teaser}
\end{figure}

This online regime has received very little attention compared to the offline counterpart. Early work~\cite{park2021changesim} in the online regime lagged far behind the offline regime in detection accuracy.
The conference version of our work, O\nobreakdash-SCD~\cite{galappaththige2025changes}, closed that gap. It carries into the online regime the central idea of our earlier offline method, MV3DCD~\cite{galappaththige2025multi}: that change can be held as a \emph{persistent property of the scene representation} rather than predicted per image individually. O\nobreakdash-SCD attaches a scalar change attribute to every primitive of a 3D Gaussian splatting (3DGS)~\cite{kerbl20233d} map of the pre-change scene representing a change field; incoming images are registered against that map, image-space cues are extracted against a rendered reference view, and those cues are fused into the per-primitive attributes. Because the attributes are shared across viewpoints, consistent evidence reinforces while view-dependent nuisance does not. This sharing yields multi-view consistency without annotation and without any assumption on the revisit trajectory. O\nobreakdash-SCD~\cite{galappaththige2025changes} was the first online method to be simultaneously pose-agnostic, label-free, multi-view, and operate at over 10\,FPS while even surpassing offline baselines.

However, that result comes at a structural cost: the fusion step of O\nobreakdash-SCD~\cite{galappaththige2025changes} is an iterative optimization. At every frame, it performs a budget of gradient steps on the change attributes, each step rendering the change field at a frame sampled from a growing buffer of previously observed frames and backpropagating a self-supervised objective against that frame's stored cue map. We refer to this re-use of stored past observations as \emph{replay}. The procedure accounts for more than half of O\nobreakdash-SCD's per-frame runtime, and the iteration budget is the only quantity available to reduce it: fewer steps buy throughput and leave the change field less converged. O\nobreakdash-SCD's accuracy and speed therefore trade off directly (\cref{fig:teaser}).

In this paper, we replace the iterative fusion and the fixed decision cutoff of O\nobreakdash-SCD with direct computations. Each observation is fused into the change field in closed form, without replay, and the decision threshold of every frame is derived exactly from O\nobreakdash-SCD's self-supervised objective, without iterative optimization. The exact threshold also gives every pixel a decision margin. The resulting method, O\nobreakdash-SCD++, is more accurate than O\nobreakdash-SCD and more than twice as fast. Three insights make this possible.

\noindent\textbf{The forward model is already known.} For a fixed camera and fixed primitive geometry and opacity, alpha compositing is linear in the per-primitive attribute being rendered~\cite{kerbl20233d,galappaththige2026predictive}. Carrying an image-space signal back onto the primitives that produced it can therefore be done with the transpose of the rendering operator instead of an iterative fit, as exploited by training-free lifting of 2D features onto Gaussians~\cite{marrie2025ludvig,joseph2025gradient}. The differentiable rasterizer already computes that transpose in a single backward pass. For change detection, one probe render and one backward call give each primitive the contribution-weighted mean of the cue over its footprint, with no iteration budget to choose.

\noindent\textbf{Replay can be avoided.} What remains is combining the per-view evidence across observed viewpoints. O\nobreakdash-SCD's replay combines them implicitly by refitting the field to every stored observation, so past observations must be kept and the estimate re-derived at each frame. Instead, O\nobreakdash-SCD++ makes the combination explicit: the support-weighted mean of the per-view evidence satisfies a two-scalar recursion, so the aggregate is reached by processing each observation once. Persistent state is two scalars per primitive, independent of sequence length.

\noindent\textbf{The change/no-change decision can be reduced to a sort.} In O\nobreakdash-SCD, descent on its self-supervised objective splits the rendered field at the midpoint: it drives unchanged surfaces towards $0$ and lifts changed surfaces above the midpoint, so a fixed cutoff of $0.5$ separates them. A field computed in closed form is not split in this way. The cutoff of $0.5$ is then no longer justified, and the label-free setting provides no annotation to tune a threshold against. For every frame, O\nobreakdash-SCD++ instead chooses the threshold whose binary change mask minimizes that same objective. Evaluated over threshold masks rather than over a continuous field, that objective depends on a mask only through the number of pixels it selects and the cue those pixels carry. The resulting one-dimensional problem is discretely convex, and one sort and one prefix scan return its \emph{exact} minimizer for every frame, without descent.

Because the decision threshold is explicit, the distance of a pixel's score to this threshold is a well-defined decision margin, and we find it strongly predictive of correctness: withholding the $10\%$ of pixels with the smallest decision margin raises intersection over union on the retained region from $0.58$ to $0.83$.
 The margin therefore ranks pixels by how reliably the mask labels them, without an auxiliary uncertainty model or held-out calibration data.

This study extends our CVPR~2026 work~\cite{galappaththige2025changes}. O\nobreakdash-SCD++
sets a new state of the art in online SCD, improving change detection
accuracy and throughput simultaneously. We believe it moves online SCD
a step closer to real-time use, where accuracy no longer has to be
traded for speed.

In summary, we make the following contributions:

\begin{enumerate}
\item \textbf{Replay-free fusion.} We show that the iterative, replay-based fusion of O\nobreakdash-SCD can be replaced by a direct estimator that is both faster and more accurate.
Each observation is transferred onto the Gaussians through the adjoint
of the renderer, weighted by reference coverage (\cref{sec:transfer}).
The transfers are combined as a support-weighted mean, maintained by a two-scalar recursion (\cref{prop:exact}). The persistent state is
two scalars per Gaussian regardless of sequence length, and the
estimator eliminates the per-frame iteration budget behind the
accuracy--speed trade-off (\cref{sec:aggregation}). As a result, fusion runs about $20\times$ faster than in O\nobreakdash-SCD.

\item \textbf{Label-free decision.} We select each frame's decision threshold with the self-supervised objective of O\nobreakdash-SCD, without labels or iterative optimization.
Evaluated over threshold masks of the fused score, the objective reduces
to a one-dimensional, discretely convex problem in the number of
selected pixels and the cue they carry (\cref{prop:convex}), whose exact
minimizer is found by one sort and one prefix scan (\cref{prop:decision}). This replaces the fixed cutoff of $0.5$ used by O\nobreakdash-SCD, which does not suit a field computed in closed form, and improves accuracy over it (\cref{sec:ablation}). The
distance of each pixel's score from the threshold gives, at no extra
cost, a decision margin that ranks pixels by confidence
(\cref{sec:margin,sec:selective}).

\item \textbf{Extended evaluation and analysis.} We extend the evaluation of
the conference version to CL-Splats and ChangeSim and to recent
baselines. We further analyse retrospective inference, showing that the final fused state is independent of frame order. We also analyse the proposed decision
margins in detail, showing that they strongly predict decision errors. Used upstream to attenuate cues, predictive photometric uncertainty~\cite{galappaththige2026predictive} improves both methods on all three datasets. On PASLCD, O\nobreakdash-SCD++ improves mIoU by 18\% at
$2.2\times$ the throughput of O\nobreakdash-SCD, and exceeds the offline result of O\nobreakdash-SCD while remaining causal (\cref{sec:experiments}).

\end{enumerate}

\section{Related Work}
\label{sec:related}

\subsection{Change Detection from Image Pairs}

SCD was classically posed as a bitemporal, pairwise learning problem: given two images of the same place captured at different times from identical~\cite{sakurada_change_2015,alcantarilla2018street,daudt2018fully,varghese2018changenet} or closely aligned~\cite{cyws2d,sakurada2020weakly,lin2025robust} viewpoints, predict the pixels that differ for a relevant reason, as a segmentation or less commonly as bounding boxes~\cite{cyws2d,cyws3d}. These methods are trained on human annotations, fully~\cite{alcantarilla2018street,daudt2018fully,cyws2d,varghese2018changenet} or partially~\cite{sakurada2020weakly,lee2024semi}, to separate the change signal from nuisance variation such as illumination shifts, shadows, or slight viewpoint discrepancies.

This paradigm has proven effective for \emph{street-view} change detection. \emph{Scene}-level change detection by an autonomous agent revisiting an environment is a different use case, and there the paradigm carries two major limitations. First, an autonomous agent under an unconstrained revisit rarely reoccupies an exact earlier viewpoint, so the near-identical-viewpoint assumption often fails, and methods trained under it degrade sharply when it is violated~\cite{galappaththige2025multi}. Second, pixelwise human annotation of change is tedious and costly; the space of possible changes in a complex environment is effectively unbounded, and supervised methods generalize poorly under the distribution shift~\cite{gulrajani2021search,galappaththige2025domain} common in real-world SCD~\cite{kim2025towards}. This second limitation has motivated recent label-free and zero-shot formulations~\cite{kim2025towards,kannanZero,cho2025zero,alpherts2025emplace} that remove the annotation requirement by exploiting visual foundation models~\cite{kirillov2023segment,oquab2023dinov2,ravi2024sam}, but retain the pairwise, near-aligned-viewpoint setting and, with it, the first limitation.

\subsection{Change Detection in Photorealistic Scene Representations}
\label{sec:related_3d}

Photorealistic scene representations~\cite{mildenhall2021nerf,kerbl20233d} changed the setting by making the reference scene queryable from arbitrary viewpoints: a representation of the pre-change scene is built offline and rendered at the post-change viewpoint, so the comparison is again between aligned images without the agent having to reoccupy a reference pose. This makes the setting \emph{pose-agnostic}, which originates in anomaly detection (AD)~\cite{zhou2023pad,kruse2024splatpose}: the reference representation is rendered at the viewpoint of the inference image, and the two images are compared. That comparison is nonetheless still pairwise, and the setting is narrower than the one SCD addresses -- single objects rather than scenes, generally a single change rather than several concurrent ones, and images captured under controlled illumination.

Our earlier MV3DCD~\cite{galappaththige2025multi}, with the concurrent
3DGS-CD~\cite{lu20253dgs}, established the premise this paper builds on:
change is a property of the scene rather than of any one image, and can
be held in the representation itself. MV3DCD attaches a change channel to
a 3DGS reconstruction of the reference scene and consolidates per-view
cues into it, so that evidence consistent across viewpoints reinforces
while inconsistent activations, likely caused by nuisance variation such
as reflections or illumination change, do not. Multi-view consistency
thus replaces the human annotation supervised methods rely on to separate
change from nuisance. MV3DCD targets both geometric change (objects added,
removed or moved) and surface-level change (appearance, spills), and
released a scene-level dataset with concurrent changes at multiple scales;
3DGS-CD addresses object rearrangement in tabletop scenes.

Subsequent work extends the same principle in other directions: object-consistent aggregation~\cite{zhou2026d}, dynamic and 4D Gaussians~\cite{wu20244d,jiang2025gaussian}, and ray-consistency constraints~\cite{huang2025c}. GS-DIFF~\cite{galappaththige2026pixels} moves the comparison itself into the primitives. PlenoCI~\cite{lai2026plenoci} proposes plenoptic features for change detection. U-3DGS~\cite{galappaththige2026predictive} predicts where a rendering of the reference map is likely to deviate from the true scene. Applied to SCD, it attenuates the change map by this per-pixel uncertainty, suppressing false positives that arise from rendering artefacts rather than from the scene. We apply it in the same way and compare it with our decision margin (\cref{sec:analysis,sec:selective}). A parallel line of zero-shot methods~\cite{friedlander2026goldilocs,wu2025scenediff} builds on feed-forward dense reconstruction backbones~\cite{wang2026pi,wang_dust3r_2024} and is oriented primarily towards geometric change, while VSCD~\cite{yoon2026vscd} treats unaligned video pairs under supervision.

All these SCD methods are offline: they require the complete post-change
capture in advance. Most register poses with
Structure-from-Motion (SfM)~\cite{schonberger2016structure} over the union of
both captures, which is defined only once both are complete. Feed-forward
methods avoid SfM but still need both captures together to place the two
reconstructions in a common frame before comparison.

\subsection{Online Scene Change Detection}

Online SCD has received comparatively little attention. The earliest work, ChangeSim~\cite{park2021changesim}, depends on an off-the-shelf RGB-D SLAM system~\cite{labbe2019rtab} for pose and assumes that the pre- and post-change trajectories are closely aligned. That assumption reduces the task to retrieving the nearest reference image by camera-centre distance, then predicting change between it and the current frame with a supervised pairwise model~\cite{sakurada2020weakly}.
O\nobreakdash-SCD~\cite{galappaththige2025changes} removes these assumptions on depth, aligned trajectories, and supervision while closing the accuracy gap between the offline and online settings. It builds on the premise established by MV3DCD~\cite{galappaththige2025multi} while removing the dependence on SfM~\cite{schonberger2016structure}, registering each incoming frame against the reference map on the fly (described in detail in \cref{sec:cues}), and improving detection accuracy by recasting MV3DCD's heuristic fusion as an optimization objective. That recasting is what lets O\nobreakdash-SCD operate online while surpassing even the strongest offline approaches, but it introduces a trade-off between speed and accuracy that \cref{sec:background} analyses. Concurrently, FastGeSCF~\cite{yeh2026towards} follows the direction of ChangeSim, retrieving the closest reference image with visual place recognition~\cite{izquierdo2024optimal} to form an image pair and detecting change with an accelerated variant of a zero-shot pairwise method~\cite{kim2025towards}. Its retrieval assumes that the revisit follows the reference trajectory closely, so that an aligned reference frame exists for each inference image, and its acceleration uses a pseudo-mask generator trained on annotated change masks~\cite{alcantarilla2018street,park2021changesim}. These choices suit its target setting, but make performance sensitive to trajectory alignment and to the training domain; \cref{tab:paslcd} shows that it does not generalise well to unconstrained trajectories under the label-free setting.

The AD methods of \cref{sec:related_3d}~\cite{zhou2023pad,kruse2024splatpose,liu2024splatpose+}
are online in the same sense as O\nobreakdash-SCD: each incoming image is registered against
the reference representation and compared against the rendered view. They are therefore pose-agnostic and label-free, but infer change independently per frame, retaining pairwise
framing. Some~\cite{zhou2023pad,kruse2024splatpose} also recover the
inference image pose by optimizing against the representation with
photometric residual, which is slow and prone to convergence failure in
large scenes containing multiple changes~\cite{galappaththige2025multi,galappaththige2025changes}.

A related but distinct body of work in robotics maintains long-term scene representations under change for navigation and manipulation, using volumetric maps~\cite{schmid2022panoptic,fu2022planesdf}, object-level tracking~\cite{qian2023pov,zhu2024living}, and more recently 3DGS~\cite{yugay2025gaussian,zeng2025gaussianupdate,ackermann2025clsplats}. Khronos~\cite{schmid2024khronos} models short- and long-term change jointly within a metric-semantic framework. These systems generally detect change at the level of objects or occupancy from depth-equipped sensors in order to keep a map usable, whereas SCD targets dense, pixel-accurate localization of change from monocular imagery, covering both geometric change and surface-level change that leaves scene geometry largely intact.

Two limitations of O\nobreakdash-SCD motivate this work. Its fusion is an iterative optimization with replay, so detection quality and throughput are coupled through the per-frame iteration budget, and the stored observations grow with the length of the revisit. Its decision is a fixed cutoff. Without labels or validation data to tune a threshold against, label-free SCD methods binarize at a fixed value~\cite{galappaththige2025multi,galappaththige2025changes,galappaththige2026pixels} or with hand-designed rules and thresholds within the pipeline~\cite{kim2025towards,lu20253dgs,lai2026plenoci}, and none reports a per-pixel confidence in its decision. Uncertainty estimates for radiance fields~\cite{galappaththige2026predictive,jiang2024fisherrf} measure the reliability of the reference map rather than of a change-detection decision. O\nobreakdash-SCD++ removes the accuracy--speed coupling by replacing the iterative, replay-based fusion with a direct estimator. Where O\nobreakdash-SCD reaches its decision through optimization and a fixed cutoff, O\nobreakdash-SCD++ computes it exactly and directly from the same objective, and obtains a per-pixel decision margin at no extra cost.

\section{Background: Online SCD with O-SCD}
\label{sec:pipeline}\label{sec:background}

O\nobreakdash-SCD~\cite{galappaththige2025changes} consists of a front end, which registers each inference image against the reference map and extracts change cues, and a change detection stage, which fuses these cues across views and decides on a binary change mask. This section states the task, the front end that O\nobreakdash-SCD++ shares with O\nobreakdash-SCD, and the iterative change detection stage of O\nobreakdash-SCD that \cref{sec:method} replaces.

\subsection{Problem Setting}
\label{sec:problem}

A reference traversal of the pre-change scene is captured and reconstructed offline into a fixed 3DGS~\cite{kerbl20233d} map
\begin{equation}
\mathcal{R} = \{(\boldsymbol{\mu}_g, \boldsymbol{\Sigma}_g, o_g, \mathbf{h}_g)\}_{g=1}^{G},
\label{eq:map}
\end{equation}
where $\boldsymbol{\mu}_g \in \mathbb{R}^3$, $\boldsymbol{\Sigma}_g$, $o_g \in [0,1]$ and $\mathbf{h}_g$ are the position, covariance, opacity and view-dependent appearance of primitive $g$, together with the reference images and their poses. During a later revisit, monocular images $I_1, I_2, \dots$ arrive sequentially from unconstrained viewpoints and the inference camera poses are \emph{not} supplied. At time $t$ the system must emit a binary change mask
\begin{equation}
B_t = \mathcal{F}\!\left(I_t, \mathcal{S}_{t-1}; \mathcal{R}\right), \qquad B_t : \Omega_t \rightarrow \{0,1\},
\label{eq:task}
\end{equation}
where $\Omega_t$ is the pixel domain of $I_t$, $\mathcal{S}_{t-1}$ is the state accumulated from earlier inference images alone, and $B_t(p)=1$ denotes that pixel $p \in \Omega_t$ shows a genuine change relative to the reference capture.

Three properties constrain any admissible $\mathcal{F}$ and are maintained throughout. \emph{Causality}: no future observation may influence $B_t$.
\emph{Label-free}: no change annotation is accessed at any stage.
\emph{Pose-agnostic}: inference extrinsics are estimated internally against $\mathcal{R}$ following~\cite{galappaththige2025changes}; this means poses need not be supplied, not that the method is pose-free. We assume that the reference traversal observes the scene from enough viewpoints, with enough coverage, to reconstruct a sufficiently accurate 3DGS map.

\subsection{Adopted Front End}
\label{sec:cues}

Registration and change-cue extraction are adopted from O\nobreakdash-SCD~\cite{galappaththige2025changes} and are summarized here for completeness.

\noindent\textbf{Registration.} Each inference image is matched against cached reference keyframes by mutual nearest neighbours of XFeat~\cite{potje2024xfeat} descriptors. The 2D observations of the reference keyframes were triangulated offline into a consistent 3D point set. The surviving 2D--3D correspondences drive PnP~\cite{lepetit2009epnp} within RANSAC~\cite{fischler1981ransac}, refined over all inliers by a GPU-parallel mini bundle adjustment~\cite{levenberg1944method,meuleman2025fly}. Rendering $\mathcal{R}$ at the recovered pose $\widehat{T}_t$ gives an aligned reference image $\Iref$ on the inference grid.

\noindent\textbf{Change cues.} The aligned pair is compared photometrically and semantically. The photometric residual follows the 3DGS loss~\cite{kerbl20233d}, combining $L_1$ with a local structural term~\cite{wang_image_2004},
\begin{equation}
D_t(p) = 0.8\,\lVert \Iinf(p) - \Iref(p)\rVert_1 + 0.2\,\bigl[1 - \mathrm{SSIM}(p)\bigr],
\label{eq:photometric}
\end{equation}
normalized per image to $P_t(p) \in [0,1]$; a frozen SAM\,2-Tiny~\cite{ravi2024sam} encoder supplies channel-averaged embedding differences, normalized and upsampled to $F_t(p) \in [0,1]$. The change cue averages the two,
\begin{equation}
c_t(p) = \tfrac{1}{2}\left[P_t(p) + F_t(p)\right] \in [0,1],
\label{eq:cue}
\end{equation}
and is the image-space evidence consumed by everything that follows.

\subsection{Iterative Change Cue Fusion in O-SCD}
\label{sec:oscd}

Given the cue map, a change detection stage must (i) transfer the per-pixel cue onto the primitives that produced those pixels, (ii) combine that transfer with everything observed so far, and (iii) decide on a binary mask. O\nobreakdash-SCD~\cite{galappaththige2025changes} performs all three by optimizing a change field. We describe it in these three steps, which \cref{sec:method} follows one by one.

\noindent\textbf{Change field and objective.} O\nobreakdash-SCD initializes a change representation $\mathcal{R}_{\mathrm{change}}$ from $\mathcal{R}$ by attaching a scalar change parameter to every primitive. Rendering that parameter at the viewpoint of frame $j$ and applying a sigmoid yields a soft mask $\widetilde{M}_j$, and the parameters are fitted by minimizing the self-supervised fusion objective
\begin{equation}
\mathcal{L}_{\mathrm{SSF}}(j) = \operatorname{mean}\!\left[2c_j \odot \bigl(1 - \widetilde{M}_j\bigr)\right] + \log\!\left(1 + \operatorname{mean}\bigl(\widetilde{M}_j\bigr)^{2}\right),
\label{eq:ssf_conf}
\end{equation}
where $c_j$ is the change cue of frame $j$ from \cref{eq:cue} and $\odot$ is the Hadamard product; the factor of two restores the $[0,2]$ scale on which the conference objective~\cite{galappaththige2025changes} is written. The first term rewards asserting change where cues are strong. The second penalizes claiming too much of the image, which excludes the trivial solution $\widetilde{M}_j \equiv 1$.

\noindent\textbf{(i) Transfer by gradient descent.} Evidence reaches the primitives only through gradients: rendering $\mathcal{R}_{\mathrm{change}}$ at viewpoint $j$, evaluating \cref{eq:ssf_conf} against $c_j$ and backpropagating through the rasterizer moves each change parameter a small step towards the cue of the pixels it covers. A single step does not transfer a cue. The field approaches it over many steps.

\noindent\textbf{(ii) Combination by replay.} Fitting only the current frame would overwrite what earlier frames established. O\nobreakdash-SCD therefore stores the cue map $c_j$ and pose $\widehat{T}_j$ of every frame observed so far, and at frame $t$ takes $b=16$ gradient steps, each on a frame $j \leq t$ sampled from this buffer. We call this re-use of stored past observations \emph{replay}. Because the parameters are shared across viewpoints, minimizing \cref{eq:ssf_conf} over the replayed frames is what enforces multi-view consistency: a cue that recurs from many viewpoints is worth asserting.

\noindent\textbf{(iii) Decision by a fixed cutoff.} After the budget is exhausted, the field is rendered at the current viewpoint and binarized at the fixed midpoint cutoff $0.5$. The sigmoid bounds the rendered field, which keeps the optimization stable, but it does not make the field a measure of confidence. The sparsity term drives unchanged surfaces to the non-negativity bound of the rendered field, while the diminishing slope of the sigmoid lets detected change settle only modestly above the cutoff. The distance from a pixel's score to the cutoff therefore reflects how far the optimizer has moved it rather than how strong the underlying evidence is. \Cref{sec:selective} returns to this point once the cutoff is derived rather than fixed.

\noindent\textbf{Time and memory complexity.}
\label{sec:cost}
Per frame, the change detection stage performs $b$ rasterizer forward--backward passes and one render at the current viewpoint, and the buffer holds one cue map and pose for every frame processed so far:
\begin{equation}
\begin{gathered}
\mathcal{T}_{\text{O\nobreakdash-SCD}} = b\,\bigl(\mathcal{T}_{\mathrm{rast}} + \mathcal{T}_{\mathrm{bwd}}\bigr) + \mathcal{T}_{\mathrm{rast}} + O(N),\\
\mathcal{M}_{\text{O\nobreakdash-SCD}} = O(G + tN),
\end{gathered}
\label{eq:complexity_oscd}
\end{equation}
where $\mathcal{T}_{\mathrm{rast}}$ and $\mathcal{T}_{\mathrm{bwd}}$ are the costs of a rasterizer forward and backward pass, $G$ is the number of primitives and $N$ the number of pixels. Two properties of this procedure determine its cost. It is \emph{iterative}: the per-frame time is the iteration budget $b$ times a forward--backward pass, so quality and throughput are coupled by construction. It is \emph{replay-based}: past cue maps must be retained and re-rendered, so memory grows with the length $t$ of the revisit and each observation is consumed many times. This stage takes $64\%$ of the per-frame processing time, nearly twice the time of pose estimation (\cref{tab:runtime}). \Cref{fig:teaser} shows this trade-off.

\subsection{Insights}
\label{sec:insights}

The trade-off stems from the estimator rather than from the problem. Each of the three steps of \cref{sec:oscd} can be reduced to a direct computation, and \cref{fig:steps} sets them side by side.

\noindent\textbf{The forward model is known and linear.} Step (i) moves evidence onto primitives by backpropagating \cref{eq:ssf_conf}, one gradient step at a time. For fixed camera, geometry and opacity, however, the map from per-primitive change attributes to a rendered change image is a known linear operator (\cref{sec:linear}). Carrying cues back onto primitives is then its transpose, which the rasterizer already computes. The transfer can therefore avoid descent (\cref{sec:transfer,sec:adjoint}).

\noindent\textbf{The multi-view combination can be computed in closed form.} Replay, step (ii), exists only to make the optimizer account for past frames, and it is the source of both the $b$ forward--backward passes and the $O(tN)$ buffer in \cref{eq:complexity_oscd}. If the quantity that replay converges towards can be written as an aggregate of per-frame transfers, and if that aggregate admits a recursion, then past frames need not be stored at all (\cref{sec:aggregation}).

\noindent\textbf{The decision is one-dimensional.} Step (iii) binarizes at a fixed cutoff that \cref{eq:ssf_conf} plays no part in choosing. Yet \cref{eq:ssf_conf} depends on a mask only through two scalar functionals of it, the number of selected pixels and the cue they carry. Evaluated over binary threshold masks rather than over a continuous field, it therefore reduces to a one-dimensional problem that can be solved exactly, which yields the threshold itself (\cref{sec:decision}).

\section[O-SCD++: Replay-Free Multi-View Evidence Fusion]{O\nobreakdash-SCD++: Replay-Free Multi-View Evidence Fusion}
\label{sec:method}

O\nobreakdash-SCD++ keeps the front end of O\nobreakdash-SCD (\cref{sec:cues}) and replaces each of the three iterative steps of its change detection stage (\cref{sec:oscd}) with a direct computation (\cref{fig:overview}). \Cref{sec:linear} reviews the linear structure of splatting that makes this possible; \cref{sec:transfer,sec:adjoint} transfer each cue map onto the primitives with one rasterizer backward pass, step (i); \cref{sec:aggregation} fuses the transfers by an exact recursion that replaces replay, step (ii); \cref{sec:prediction} forms the per-pixel score; and \cref{sec:decision} derives the exact threshold, step (iii).

\subsection{Splatting as a Linear Operator}
\label{sec:linear}

A rendered pixel is a blend of the primitives along its ray. We first make the blending weights explicit, since everything that follows uses them. Fix the camera at the estimated pose $\widehat{T}_t$ and fix the geometry and opacity of $\mathcal{R}$. Let $\alpha_{tg}(p)$ denote the screen-space opacity that primitive $g$ contributes at pixel $p$, and let $\prec_t$ order primitives front to back. Alpha compositing assigns to $g$ the share
\begin{equation}
A_t(p,g) = \alpha_{tg}(p)\!\!\prod_{h \prec_t g}\!\!\bigl[1 - \alpha_{th}(p)\bigr],
\qquad
a_t(p) = \sum_{g=1}^{G} A_t(p,g),
\label{eq:compositing}
\end{equation}
so that $A_t(p,g)$ is the fraction of pixel $p$ that primitive $g$ is responsible for painting, and the \emph{coverage} $a_t(p) \in [0,1]$ is the total fraction that the reference map explains at all. Coverage is near one wherever the reconstruction is complete, or near zero where the ray escapes to background.

\begin{figure}[t]
\centering
\definecolor{stours}{RGB}{252,232,232}
\definecolor{stoursd}{RGB}{214,39,40}
\definecolor{stold}{RGB}{222,235,247}
\definecolor{stolddraw}{RGB}{31,119,180}
\resizebox{\columnwidth}{!}{%
\begin{tikzpicture}[
    font=\footnotesize,
    stcell/.style={draw, rounded corners=2pt, align=left, text width=3.35cm, minimum height=1.05cm, inner sep=3pt},
    stoldbox/.style={stcell, fill=stold, draw=stolddraw!60},
    stnewbox/.style={stcell, fill=stours, draw=stoursd!60},
    stlabel/.style={font=\footnotesize\bfseries, align=center, text width=1.2cm},
    sthead/.style={font=\footnotesize\bfseries, align=center},
    arr/.style={-{Stealth[length=4pt]}, line width=0.6pt, draw=black!60},
]
\def\dy{1.3}
\node[sthead, text=stolddraw] at (0,0.85) {O-SCD (iterative)};
\node[sthead, text=stoursd] at (4.3,0.85) {O-SCD++ (direct)};
\node[stlabel] at (-2.75,0) {(i)\\[-1pt]{\scriptsize\mdseries transfer}};
\node[stoldbox] (o1) at (0,0) {gradient steps on \cref{eq:ssf_conf}, backpropagated through the rasterizer};
\node[stnewbox] (n1) at (4.3,0) {one rasterizer backward pass applies $\Amat^{\tr}$ (\cref{eq:adjoint})};
\node[stlabel] at (-2.75,-\dy) {(ii)\\[-1pt]{\scriptsize\mdseries combine}};
\node[stoldbox] (o2) at (0,-\dy) {replay: $b{=}16$ steps per frame on stored cue maps $c_1,\dots,c_t$};
\node[stnewbox] (n2) at (4.3,-\dy) {exact recursion on two scalars $(q_g, W_g)$ (\cref{prop:exact})};
\node[stlabel] at (-2.75,-2*\dy) {(iii)\\[-1pt]{\scriptsize\mdseries decide}};
\node[stoldbox] (o3) at (0,-2*\dy) {fixed cutoff $0.5$;\\ no decision confidence};
\node[stnewbox] (n3) at (4.3,-2*\dy) {exact minimizer of \cref{eq:ssf_conf} by sort and scan (\cref{prop:decision}); margin $\mu_t$};
\draw[black!30] (-3.35,-2.5*\dy-0.12) -- (6.15,-2.5*\dy-0.12);
\node[stlabel] at (-2.75,-3*\dy+0.1) {cost\\[-1pt]{\scriptsize\mdseries per frame}};
\node[align=center, text width=3.35cm] (o4) at (0,-3*\dy+0.1) {$b(\mathcal{T}_{\mathrm{rast}}{+}\mathcal{T}_{\mathrm{bwd}}) + \mathcal{T}_{\mathrm{rast}}$\\ memory $O(G + tN)$};
\node[align=center, text width=3.35cm] (n4) at (4.3,-3*\dy+0.1) {$2\mathcal{T}_{\mathrm{rast}}{+}\mathcal{T}_{\mathrm{bwd}}{+}O(N\log N)$\\ memory $O(G + N)$};
\foreach \i in {1,2,3} \draw[arr] (o\i.east) -- (n\i.west);
\end{tikzpicture}}
\caption{The three steps of the change detection stage in O\protect\nobreakdash-SCD~\cite{galappaththige2025changes} and in O\protect\nobreakdash-SCD++. Each iterative step of O\protect\nobreakdash-SCD is replaced by a direct computation; the per-frame time loses its iteration budget $b$, and the memory no longer grows with the number $t$ of frames processed.}
\label{fig:steps}
\end{figure}
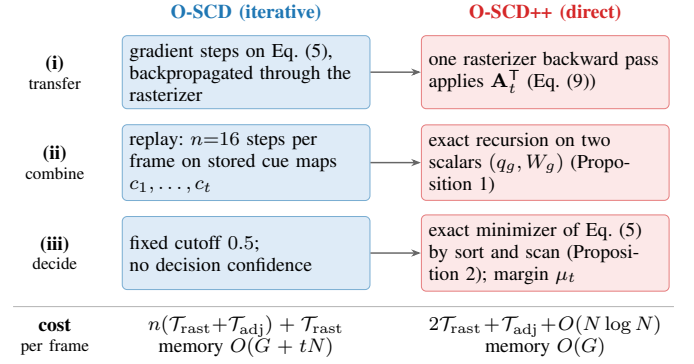

\begin{figure*}[t]
\centering
\input{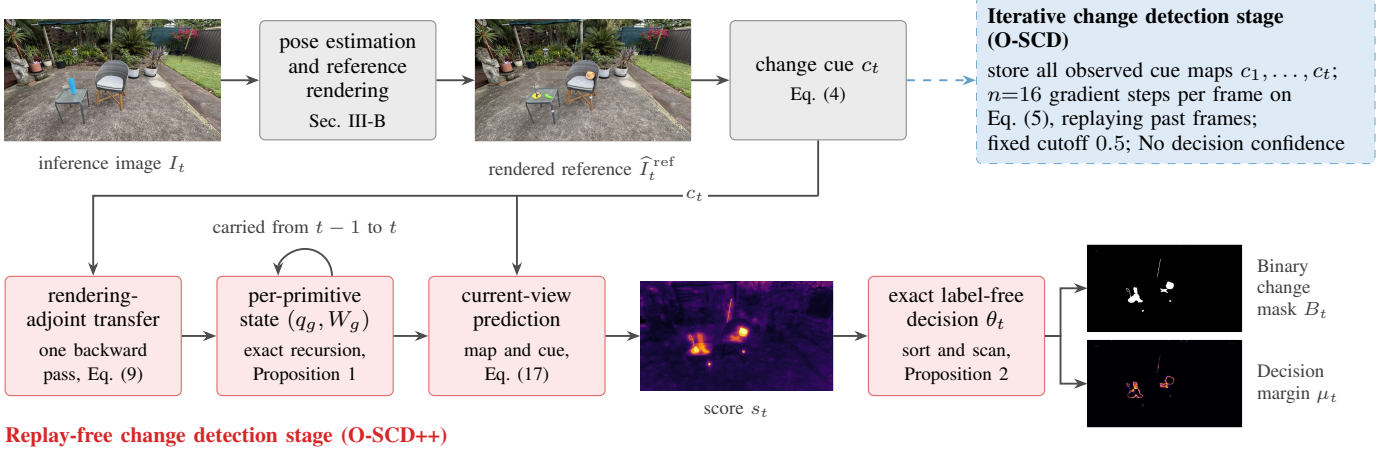}
\caption{Overview of O\protect\nobreakdash-SCD++ for one inference frame. The front end, adopted from O\protect\nobreakdash-SCD~\cite{galappaththige2025changes}, registers the inference image against the reference map and computes a change cue. The replay-free change detection stage transfers the cue onto the reference primitives with one rasterizer backward pass, folds it into a two-scalar state per primitive by an exact recursion, predicts a score for the current view from the updated state and the current cue, and binarizes it with the exact minimizer of a label-free objective, which also yields a per-pixel decision margin. It eliminates the replay-based optimization of O\protect\nobreakdash-SCD (dashed).}
\label{fig:overview}
\end{figure*}

Because the composite of any per-primitive scalar field $\mathbf{f} \in \mathbb{R}^{G}$ is $\sum_g A_t(p,g)\,f_g$, rendering is a \emph{linear} operator in that field once geometry and opacity are fixed,
\begin{equation}
\mathbf{r}_t = \Amat\mathbf{f}, \qquad \Amat \in \mathbb{R}^{N \times G},
\qquad \mathbf{a}_t = \Amat\mathbf{1}.
\label{eq:linear}
\end{equation}
Coverage is thus the rendered image of the all-ones field. This linearity follows directly from the compositing rule of splatting~\cite{kerbl20233d,zwicker2002ewa}. Recent work uses it to fit per-primitive photometric uncertainty by solving a linear system~\cite{galappaththige2026predictive}, and training-free lifting of 2D features onto Gaussians exploits the transpose of the rendering operator~\cite{marrie2025ludvig,joseph2025gradient}. We use this transpose to carry change evidence onto primitives and to fuse it online. We never form $\Amat$. It is used only as an operator, and \cref{sec:adjoint} shows that both it and its transpose $\Amat^{\tr}$ are already supplied by the differentiable rasterizer.

\subsection{Evidence Transfer by the Rendering Transpose}
\label{sec:transfer}

We want every primitive to carry a running estimate of how changed the surface it represents is, built from cues carried back onto it by the transpose $\Amat^{\tr}$ of \cref{eq:linear}.

\noindent\textbf{Transferring cues onto primitives.} Applied to a cue map, $\Amat^{\tr}$ accumulates, for each primitive, the cue of every pixel it helped paint, weighted by its contribution.

We additionally weight each pixel by its coverage $a_t(p)$, the fraction of pixel $p$ that the reference explains. Where coverage is low, the rendered reference shows mostly background, and a high cue there signals a missing reconstruction rather than a change. Each pixel therefore counts in proportion to how much of it the primitive painted \emph{and} how much of it the reference explains:
\begin{equation}
\mathbf{n}_t = \Amat^{\tr}\bigl(\mathbf{a}_t \odot \mathbf{c}_t\bigr),
\quad
\mathbf{m}_t = \Amat^{\tr}\mathbf{a}_t,
\quad
e_{g,t} = \frac{n_{g,t}}{m_{g,t}}.
\label{eq:adjoint}
\end{equation}
Here $m_{g,t}$ is the observation mass of primitive $g$, the coverage-weighted area of the pixels it paints, which measures how much of $g$ the frame observes; $n_{g,t}$ is the accumulated cue over those same pixels. Their ratio $e_{g,t}$ is the evidence of view $t$ about $g$, and is set to zero wherever $m_{g,t}$ vanishes. Written out, the per-primitive evidence is
\begin{equation}
e_{g,t} = \frac{\sum_{p} A_t(p,g)\,a_t(p)\,c_t(p)}{\sum_{p} A_t(p,g)\,a_t(p)} \in [0,1].
\label{eq:evidence_explicit}
\end{equation}
This is the contribution-weighted mean cue over the primitive's footprint, in which $a_t$ weights each observation rather than contributing a value to be averaged.

Three properties follow. A primitive is not judged more changed merely for being larger, since its footprint enters numerator and normalizer alike. Primitives the reference itself occludes receive little weight automatically, since their compositing weights are small behind an opaque surface. And because the cue lies in $[0,1]$ and the weights are non-negative, $e_{g,t}$ is a weighted average of cue values and so stays in $[0,1]$. The coverage weight matters only where coverage is partial, at reconstruction boundaries and in sparsely reconstructed regions, since a pixel the reference does not cover at all has $A_t(p,g)=0$ for every $g$ and drops out of both sums.

\subsection{Computing the Transfer with the Existing Rasterizer}
\label{sec:adjoint}

The accumulated cue $\mathbf{n}_t$ and the observation mass $\mathbf{m}_t$ of \cref{eq:adjoint} are both obtained without materializing $\Amat$. Rendering is linear in the composited attribute (\cref{eq:linear}), so the Jacobian of the render with respect to that attribute is $\Amat$ itself, and a backward pass through the rasterizer~\cite{kerbl20233d} therefore applies $\Amat^{\tr}$.

Let each primitive carry a temporary spherical-harmonic coefficient $\theta_g$ at degree zero, whose map to emitted scalar value is affine, $y_g = \gamma\theta_g + \tfrac{1}{2}$ with $\gamma = 1/\sqrt{4\pi}$. Compositing these values gives the rendered scalar image $\mathbf{r}_t$ of \cref{eq:linear},
\begin{equation}
r_t(p) = \gamma\,\bigl(\Amat\boldsymbol{\theta}\bigr)(p) + \tfrac{1}{2}\,a_t(p),
\label{eq:affine_render}
\end{equation}
where the second term is the constant $\tfrac{1}{2}$ emitted by every primitive, composited with total weight $a_t(p)$. Two consequences follow. Setting $\boldsymbol{\theta}=\mathbf{0}$ makes every primitive emit $\tfrac{1}{2}$, so this \emph{probe} render returns coverage directly,
\begin{equation}
\mathbf{a}_t = 2\,\mathbf{r}^{\mathrm{probe}}_t .
\label{eq:probe}
\end{equation}
And for any image-space cotangent $\mathbf{v}$, the vector--Jacobian product of \cref{eq:affine_render} with respect to $\boldsymbol{\theta}$ is $\gamma\,\Amat^{\tr}\mathbf{v}$. We supply $\mathbf{v}_1 = \mathbf{a}_t \odot \mathbf{c}_t$, the coverage-weighted cue, and $\mathbf{v}_2 = \mathbf{a}_t$, the coverage, in two of the rasterizer's colour channels of a single backward call. It returns $\gamma\mathbf{n}_t$ and $\gamma\mathbf{m}_t$, the accumulated cue and observation mass of \cref{eq:adjoint} scaled by $\gamma$. Their ratio cancels $\gamma$. One probe forward pass and one backward call therefore deliver coverage, accumulated cue and observation mass together, at the cost of a single rendered frame.

Despite invoking automatic differentiation, this is a single linear transfer, not iterative optimization: no learning rate or iteration budget exists. No previously observed image, cue map, or pose participates, and each observation is transferred exactly once. We note for precision that the method is not gradient-free -- it executes a rasterizer backward pass -- but it performs no gradient \emph{descent}.

\subsection{Recursive Multi-View Aggregation}
\label{sec:aggregation}

The key result of this subsection is that multi-view fusion can be performed without replay (step (ii) of \cref{sec:oscd}). Once every view has been transferred to per-primitive evidence and the mass behind it, the fused estimate is a support-weighted mean over all views observed so far, which an exact two-scalar recursion maintains. No past view has to be stored or revisited. Each view supplies evidence $e_{g,t}$ about each primitive it sees, together with the observation mass $m_{g,t}$ behind it.

\noindent\textbf{From mass to a bounded weight.} Weighting each view by the observation mass $m_{g,t}$ would let the views in which $g$ covers the most pixels dominate its estimate. We instead let every view that clearly observes $g$ count about equally, and views in which $g$ barely contributes count less,
\begin{equation}
\rho_{g,t} = \frac{m_{g,t}}{m_{g,t}+1} \in [0,1).
\label{eq:rho}
\end{equation}
Unobserved primitives receive $\rho_{g,t}=0$ and are left untouched. Weighting by the raw mass instead, $\rho_{g,t}=m_{g,t}$, gives the all-view mean computed by training-free lifting~\cite{marrie2025ludvig,joseph2025gradient}, which \cref{sec:ablation} compares against.

\noindent\textbf{The aggregate and its recursion.} The natural multi-view estimate is the support-weighted mean of the per-view evidence over all views observed so far,
\begin{equation}
q_{g,t} = \frac{\sum_{j \leq t}\rho_{g,j}\,e_{g,j}}{\sum_{j \leq t}\rho_{g,j}}.
\label{eq:exact}
\end{equation}
Each primitive stores only its current estimate $q_g$ and the accumulated support $W_g$ behind it, both initialized to zero, and folds in new evidence as
\begin{align}
W_{g,t} &= W_{g,t-1} + \rho_{g,t}, \qquad \eta_{g,t} = \frac{\rho_{g,t}}{W_{g,t}}, \label{eq:update_w}\\
q_{g,t} &= \bigl(1-\eta_{g,t}\bigr)\,q_{g,t-1} + \eta_{g,t}\,e_{g,t}. \label{eq:update_q}
\end{align}
Here $\eta_{g,t}$ is taken as zero when $W_{g,t}=0$, the case of a primitive that no view has yet rendered, which leaves the state at its initialization.

\begin{proposition}[Exactness of the recursion]
\label{prop:exact}
Let $\rho_{g,j} \geq 0$ and $e_{g,j} \in [0,1]$ for all $j$. Then for every $t$ with $W_{g,t} > 0$, the state produced by \cref{eq:update_w,eq:update_q} from $q_{g,0}=W_{g,0}=0$ satisfies \cref{eq:exact} exactly. Consequently $q_{g,t} \in [0,1]$, the persistent state is two scalars per primitive, and no past observation is required.
\end{proposition}

\begin{proof}
Write $S_{g,t} = \sum_{j\leq t}\rho_{g,j}e_{g,j}$ and $W_{g,t} = \sum_{j\leq t}\rho_{g,j}$, and proceed by induction on $t$. For the first view with $\rho_{g,t}>0$ we have $W_{g,t}=\rho_{g,t}$, hence $\eta_{g,t}=1$ and $q_{g,t}=e_{g,t}=S_{g,t}/W_{g,t}$. Assume $q_{g,t-1}=S_{g,t-1}/W_{g,t-1}$. Then, using $W_{g,t}=W_{g,t-1}+\rho_{g,t}$,
\[
q_{g,t}
= \frac{W_{g,t}-\rho_{g,t}}{W_{g,t}}\cdot\frac{S_{g,t-1}}{W_{g,t-1}} + \frac{\rho_{g,t}}{W_{g,t}}e_{g,t}
= \frac{S_{g,t-1}+\rho_{g,t}e_{g,t}}{W_{g,t}},
\]
which is $S_{g,t}/W_{g,t}$. Views with $\rho_{g,t}=0$ leave both $W_g$ and $q_g$ unchanged and also leave \cref{eq:exact} unchanged. Since \cref{eq:exact} is a convex combination of values in $[0,1]$, $q_{g,t}\in[0,1]$.
\end{proof}

\Cref{prop:exact} is what makes the formulation replay-free. Where O\nobreakdash-SCD refits every stored cue map at each frame, the fused estimate $q_{g,t}$ of \cref{eq:exact} is reached in closed form from two scalars per primitive and the current view alone. The recursion is the standard incremental form of a weighted mean~\cite{curless1996volumetric}, and it applies here because the transfer of \cref{sec:transfer} gives every view its own per-primitive evidence.

\begin{remark}[An exact aggregate, not an approximation]
\Cref{prop:exact} should not be read as claiming that the recursion approximates the replay-based optimum of \cref{eq:ssf_conf}. The two estimators are different, and we state ours explicitly: \cref{eq:exact} is the support-weighted mean of the per-view evidence. The claim is that this aggregate is well motivated, is computed exactly and, as \cref{sec:experiments} shows, is also more accurate.
\end{remark}

\subsection{Current-View Prediction}
\label{sec:prediction}

Rendering the updated attributes at the inference pose composites them back to image, $(\Amat\mathbf{q}_t)(p) = \sum_g A_t(p,g)\,q_{g,t}$. Since the compositing weights sum to coverage and $q_{g,t}\in[0,1]$, this render lies in $[0,a_t(p)]$. Scaling the current cue the same way and averaging the two gives the fused score
\begin{equation}
s_t(p) = \tfrac{1}{2}\Bigl[(\Amat\mathbf{q}_t)(p) + a_t(p)\,c_t(p)\Bigr].
\label{eq:score}
\end{equation}
Neither term can assert change where the reference has no surface, since both are bounded by $a_t(p)$. Coverage therefore enters the method twice with different meanings: in \cref{eq:adjoint} it decides which pixels count as a valid observation \emph{of a primitive}, and in \cref{eq:score} it measures how far a pixel admits a reference-based comparison. \Cref{sec:ablation} empirically shows that the two terms together outperform either alone.

\subsection{Exact Label-Free Decision}
\label{sec:decision}

The score $s_t$ is continuous, while the required output \cref{eq:task} is a binary change/no-change mask. The fixed cutoff of $0.5$ is no longer justified: it suited O\nobreakdash-SCD only because descent splits its rendered field at the midpoint, and a field computed in closed form is not split in this way. For every frame, we instead choose the threshold whose mask minimizes the self-supervised objective that O\nobreakdash-SCD minimizes by descent, \cref{eq:ssf_conf}.

\noindent\textbf{The objective as a function of the threshold.} For a threshold $\tau$, let $M_\tau(p) = \mathds{1}\bigl[s_t(p) > \tau\bigr]$ be the binary change mask that selects every pixel whose score exceeds $\tau$, where $\mathds{1}[\cdot]$ is one when its argument holds and zero otherwise, and let $k(\tau) = \sum_p M_\tau(p)$ be the number of selected pixels in a frame of $N$ pixels. We evaluate \cref{eq:ssf_conf} with the fused score $s_t$ in the role of the cue. Since \cref{eq:ssf_conf} applies the sigmoid $\sigma$ to the rendered field, a threshold mask enters it as $\widetilde{M} = \sigma(M_\tau) = \sigma_0 + \Delta M_\tau$, with $\sigma_0 = \sigma(0)$, $\sigma_1 = \sigma(1)$ and $\Delta = \sigma_1 - \sigma_0$, which gives
\begin{equation}
\begin{split}
\mathcal{L}_t(\tau) = {} & \frac{1}{N}\sum_{p} 2\,s_t(p)\bigl[1 - \sigma_0 - \Delta\,M_\tau(p)\bigr] \\
& + \log\!\left(1 + \Bigl(\sigma_0 + \Delta\tfrac{k(\tau)}{N}\Bigr)^{2}\right).
\end{split}
\label{eq:ssf_mask}
\end{equation}
The first term rewards selecting pixels with high scores and the second penalizes selecting too many. The decision of frame $t$ is the mask at its minimizer,
\begin{equation}
\tau_t \in \operatorname*{arg\,min}_{\tau}\, \mathcal{L}_t(\tau), \qquad B_t = M_{\tau_t}.
\label{eq:mask}
\end{equation}

\noindent\textbf{Exact minimization.} $\mathcal{L}_t$ depends on $\tau$ only through the number of pixels the mask selects and the total score they carry, and it changes only when $\tau$ crosses a score. \Cref{eq:mask} therefore reduces to choosing how many pixels to select, a one-dimensional problem with an exact solution that requires no descent.

\begin{proposition}[Exactness of the decision]
\label{prop:decision}
Sort the scores in decreasing order, $s_{(1)} \geq \cdots \geq s_{(N)}$, where $s_{(i)}$ is the $i$-th largest, and let $Q_k = \sum_{i \leq k} 2\,s_{(i)}$, with $Q_0 = 0$, be twice the summed score of the $k$ highest-scoring pixels. Let
\begin{equation}
J_t(k) = -\frac{\Delta}{N}\,Q_k + \log\!\left(1 + \Bigl(\sigma_0 + \Delta\tfrac{k}{N}\Bigr)^{2}\right)
\label{eq:scan}
\end{equation}
be \cref{eq:ssf_mask} for the mask that selects these $k$ pixels, with the term that does not depend on the mask dropped, and let $\mathcal{K}_t = \{0, N\} \cup \bigl\{k : s_{(k)} > s_{(k+1)}\bigr\}$. Then
\begin{equation}
\min_{\tau} \mathcal{L}_t(\tau) = \min_{k \in \mathcal{K}_t} J_t(k) + \text{const},
\end{equation}
and the minimum is attained at $\tau = s_{(k^{*}+1)}$, with $k^{*} \in \operatorname*{arg\,min}_{k \in \mathcal{K}_t} J_t(k)$ and $s_{(N+1)} = -\infty$.
\end{proposition}

\begin{proof}
Expanding \cref{eq:ssf_mask}, the mask enters only through $\sum_p 2\,s_t(p)M_\tau(p)$ and its cardinality $k = k(\tau)$:
\begin{equation*}
\begin{split}
\mathcal{L}_t(\tau) = {} & \underbrace{\frac{1-\sigma_0}{N}\sum_p 2\,s_t(p)}_{\text{independent of } \tau} \;-\; \frac{\Delta}{N}\sum_p 2\,s_t(p)\,M_\tau(p) \\
& +\; \log\!\left(1 + \Bigl(\sigma_0+\Delta\tfrac{k}{N}\Bigr)^{2}\right).
\end{split}
\end{equation*}
\emph{Fix $k$.} The third term is then constant, and since $\Delta>0$ the second is minimized by collecting as much score as possible -- which $k$ pixels do precisely when they are the $k$ largest, with total $Q_k$. A threshold mask with $k(\tau)=k$ selects exactly these pixels, so its objective is $J_t(k) + \text{const}$, and the search over $\tau$ collapses to a search over $k$ alone.

\emph{Then vary $k$.} A threshold selects the $k$ largest scores unless it would have to split a tie at position $k$, so the counts some $\tau$ can produce are exactly $\mathcal{K}_t$, and $\tau = s_{(k+1)}$ produces count $k$ for every $k \in \mathcal{K}_t$. Evaluating $J_t$ on $\mathcal{K}_t$ is therefore exhaustive. Sorting costs $O(N\log N)$; the cumulative sum, the tie test and the scan are each $O(N)$.
\end{proof}

\begin{proposition}[Discrete convexity]
\label{prop:convex}
$J_t$ is discretely convex on $\{0,1,\dots,N\}$: $J_t(k+1) - 2J_t(k) + J_t(k-1) \geq 0$ for all $0 < k < N$. Its minimizers form a contiguous interval.
\end{proposition}

\begin{proof}
Both terms are convex in $k$. The increments of $Q_k$ are the doubled sorted scores $2s_{(k)}$, which are non-increasing, so $Q_k$ is discretely concave and $-\tfrac{\Delta}{N}Q_k$ discretely convex. For the second, $g(x) = \log(1+x^2)$ has $g''(x) = 2(1-x^{2})/(1+x^{2})^{2} > 0$ whenever $|x| < 1$; its argument $x_k = \sigma_0 + \Delta k/N$ is affine in $k$ and lies between $\sigma_0$ and $\sigma_1$, which are sigmoid outputs and so below one. So $g$ is evaluated only where it is convex, and a sum of discretely convex functions is discretely convex.
\end{proof}

The scan of \cref{prop:decision} is exhaustive, so convexity is not what makes it exact; what it adds is that the threshold is well defined up to ties, and that the minimum could instead be located by ternary search in $O(\log N)$ evaluations. We retain the linear scan, since the sort already dominates the cost and the scan is a single fused pass on the GPU. We set $\tau_t = s_{(k^{*}+1)}$ with $k^{*}$ the smallest minimizer, which the scan returns. No ground-truth mask or hand-set cutoff enters the decision at any point.

It is worth being precise about what \cref{prop:decision} does and does not claim. The minimum is global with respect to \cref{eq:ssf_mask} and within the threshold masks $M_\tau$; it is not a claim about the evaluation metric. \Cref{eq:ssf_mask} is a label-free surrogate, so a threshold selected against ground-truth masks can score higher, and a mask outside the threshold family can score higher still. \Cref{sec:ablation} reports an oracle threshold, chosen per frame to maximise intersection over union, as an upper bound. This gap is the price of deciding without annotation under a label-free setting.

\subsection{Decision Margins}
\label{sec:margin}

The decision margin of a pixel is the distance of its score to the threshold,
\begin{equation}
\mu_t(p) = \bigl\lvert s_t(p) - \tau_t \bigr\rvert,
\label{eq:margin}
\end{equation}
and is available at no additional cost once $\tau_t$ is known. Pixels with a small margin are those a slightly different threshold would have labelled differently.

We note that $\mu_t$ does not modify $B_t$: \cref{eq:mask} is emitted unchanged, and the margin is an additional output. It is a selective-prediction signal~\cite{elyaniv2010foundations,geifman2017selective}: withholding the pixels of smallest margin raises accuracy on those that remain, at the cost of labelling fewer of them. Because it is a ranking rather than a calibrated probability, we assess it as such. \Cref{sec:selective} reports the sparsification behaviour -- how detection quality on the retained region improves as the least confident pixels are withheld -- against random withholding as a lower reference and an oracle ordering as an upper one.

\subsection{Complexity and Properties}
\label{sec:complexity}

Per frame, the change detection stage performs one probe render, one rasterizer backward call, one render of the updated attributes, $O(G)$ primitive-wise updates, and the sort of \cref{eq:scan}:
\begin{equation}
{\begin{gathered}
\mathcal{T}_{\text{O\nobreakdash-SCD++}} = 2\,\mathcal{T}_{\mathrm{rast}} + \mathcal{T}_{\mathrm{bwd}} + O(G) + O(N\log N),\\
\mathcal{M}_{\text{O\nobreakdash-SCD++}} = O(G + N).
\end{gathered}}
\label{eq:complexity}
\end{equation}
In memory, the $O(G)$ term is the two-scalar state of every primitive, which persists across frames, and the $O(N)$ term holds the maps of the current frame, which are discarded once it is processed. Two properties distinguish \cref{eq:complexity} from the cost of O\nobreakdash-SCD in \cref{eq:complexity_oscd}. It carries \emph{no multiplicative factor} for an optimization budget, so there is no quantity to trade detection quality against. And its memory does not \emph{grow} with the number $t$ of frames already processed, because no past frame is stored; each observation is used exactly once, whereas O\nobreakdash-SCD spreads a fixed budget over a growing buffer. \Cref{alg:main} summarizes the per-frame computation of O\nobreakdash-SCD++.

\begin{algorithm}[t]
\caption{O\protect\nobreakdash-SCD++ for one inference frame}
\label{alg:main}
\begin{algorithmic}[1]
\REQUIRE fixed map $\mathcal{R}$; state $q_g = W_g = 0$ for all $g$
\STATE estimate the pose $\widehat{T}_t$ of $I_t$ \hfill{\scriptsize front end}
\STATE render $\mathcal{R}$ at $\widehat{T}_t$ to obtain $\Iref$ \hfill{\scriptsize front end}
\STATE compute the cue $c_t$ from $(\Iinf, \Iref)$ \hfill{\scriptsize(\ref{eq:cue})}
\STATE probe render with $\boldsymbol{\theta}=\mathbf{0}$; $\;\mathbf{a}_t \leftarrow 2\,\mathbf{r}^{\mathrm{probe}}_t$ \hfill{\scriptsize(\ref{eq:probe})}
\STATE one backward call on $(\mathbf{a}_t\!\odot\!\mathbf{c}_t,\, \mathbf{a}_t)$ $\Rightarrow$ $\mathbf{n}_t, \mathbf{m}_t$ \hfill{\scriptsize(\ref{eq:adjoint}),(\ref{eq:affine_render})}
\STATE $e_{g,t} \leftarrow n_{g,t}/m_{g,t}$; $\;\rho_{g,t} \leftarrow m_{g,t}/(m_{g,t}\!+\!1)$ \hfill{\scriptsize(\ref{eq:adjoint}),(\ref{eq:rho})}
\STATE $W_g \leftarrow W_g + \rho_{g,t}$; $\;q_g \leftarrow q_g + \tfrac{\rho_{g,t}}{W_g}(e_{g,t}-q_g)$ \hfill{\scriptsize(\ref{eq:update_w}),(\ref{eq:update_q})}
\STATE $s_t \leftarrow \tfrac{1}{2}\bigl[\Amat\mathbf{q} + \mathbf{a}_t\!\odot\!\mathbf{c}_t\bigr]$ \hfill{\scriptsize(\ref{eq:score})}
\STATE $k^{*} \leftarrow \operatorname{arg\,min}_{k\in\mathcal{K}_t} J_t(k)$; $\;\tau_t \leftarrow s_{(k^{*}+1)}$ \hfill{\scriptsize(\ref{eq:scan})}
\RETURN $B_t = M_{\tau_t}$ and margin $\mu_t$ \hfill{\scriptsize(\ref{eq:mask}),(\ref{eq:margin})}
\end{algorithmic}
\end{algorithm}

\section{Experiments}
\label{sec:experiments}

\noindent\textbf{Datasets.} The goal of this work is scene-level change detection in complex real-world scenes that contain multiple concurrent changes. Following O\nobreakdash-SCD~\cite{galappaththige2025changes}, we therefore evaluate primarily on PASLCD~\cite{galappaththige2025multi}, a real-world multi-view SCD dataset of ten scenes comparable in extent to those of Mip-NeRF360~\cite{Barron2021MipNeRF3U}. Each scene contains multiple concurrent changes, both geometric and surface-level and at different scales, and is revisited twice, once under lighting similar to the reference capture and once under different lighting, giving 20 instances. The scenes cover indoor and outdoor environments and both 360$^\circ$ and front-facing capture. We further extend evaluation to CL-Splats~\cite{ackermann2025clsplats}, whose five scenes are smaller in extent than those of PASLCD and each contain a single object-level geometric change, and to ChangeSim~\cite{park2021changesim}, a simulated industrial warehouse dataset with geometric changes only.

\noindent\textbf{Baselines and Metrics.} We compare against publicly available online SCD methods including recent FastGeSCF~\cite{yeh2026towards}. These comprise ChangeSim~\cite{park2021changesim}; two baselines that pair ChangeSim's frame retrieval with the strongest pairwise detectors, CYWS2D~\cite{cyws2d} and GeSCF~\cite{kim2025towards}; the pose-agnostic AD methods OmniposeAD~\cite{zhou2023pad}, SplatPose~\cite{kruse2024splatpose} and SplatPose+~\cite{liu2024splatpose+}; and the conference version, O\nobreakdash-SCD~\cite{galappaththige2025changes}. Since neither PASLCD nor CL-Splats provides depth for the RGB-D SLAM front end of ChangeSim, its frame retrieval is given ground-truth poses. Supervised methods~\cite{yeh2026towards,cyws2d} use the publicly released checkpoints. We report O\nobreakdash-SCD as reproduced with its released code, which runs slightly faster than published ($11.8$ vs $11.2$\,FPS) at comparable accuracy ($0.489$ against $0.486$ mIoU).
Following O\nobreakdash-SCD~\cite{galappaththige2025changes} and prior work~\cite{galappaththige2025multi, lin2025robust, zhou2026d, huang2025c}, we report mean intersection over union (mIoU) and F1 score computed for change pixels in the ground-truth mask over the views. FPS denotes the number of frames that can be inferred per second of wall-clock time. All experiments are conducted on an NVIDIA RTX 4090 GPU.

\subsection{Comparison with the State of the Art}
\label{sec:comparison}

\begin{figure*}[t]
\centering
\definecolor{qualfn}{RGB}{232,53,43}
\definecolor{qualfp}{RGB}{200,110,215}
\newcommand{\qualtile}[1]{\includegraphics[width=0.238\textwidth]{figures/qualitative/#1}}
\newcommand{\qualmask}[3]{%
\begin{tikzpicture}
\node[inner sep=0pt, anchor=south west] (img) at (0,0) {\qualtile{#1}};
\node[anchor=north east, inner sep=1.5pt, text=white, fill=black, fill opacity=0.6, text opacity=1, font=\fontsize{6}{7}\selectfont, align=right]
    at (img.north east) {IoU: #2\\F1: #3};
\end{tikzpicture}}
\newcommand{\qualrow}[7]{%
\rotatebox{90}{\parbox{0.134\textwidth}{\centering\scriptsize #1}} &
\qualtile{#2_ref.jpg} & \qualtile{#2_inf.jpg} &
\qualmask{#2_oscd.png}{#3}{#4} & \qualmask{#2_ours.png}{#5}{#6} \\[#7]}
\newcommand{\qualkey}[2]{\fcolorbox{black!60}{#1}{\rule{0pt}{3pt}\rule{9pt}{0pt}}~#2}
\setlength{\tabcolsep}{0.6pt}
\renewcommand{\arraystretch}{0}
\begin{tabular}{ccccc}
 & \footnotesize Pre-change (rendered) & \footnotesize Post-change & \footnotesize O-SCD~\cite{galappaththige2025changes} & \footnotesize O-SCD++ \\[3pt]
\qualrow{Porch}{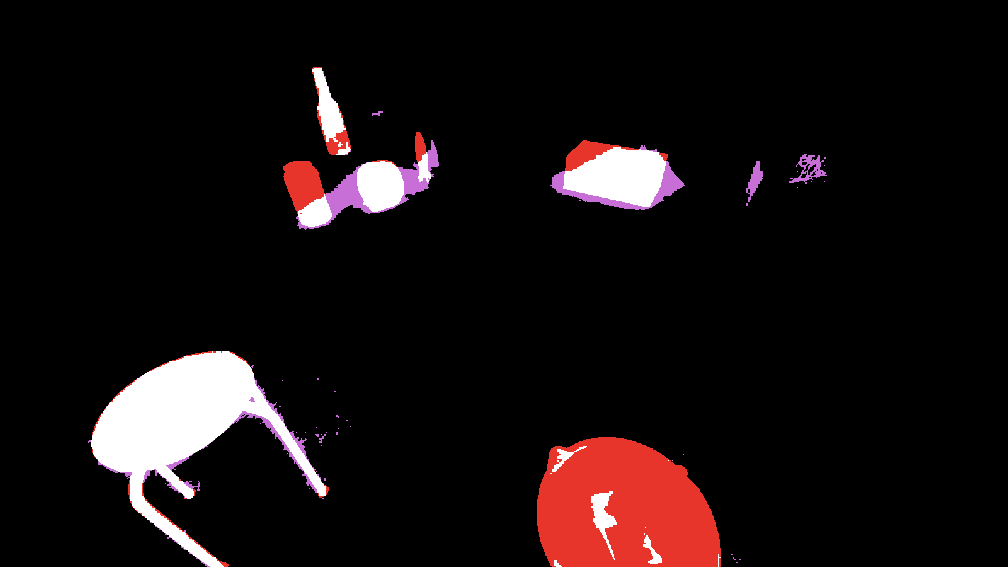}{0.486}{0.654}{0.815}{0.898}{1.2pt}
\qualrow{Zen}{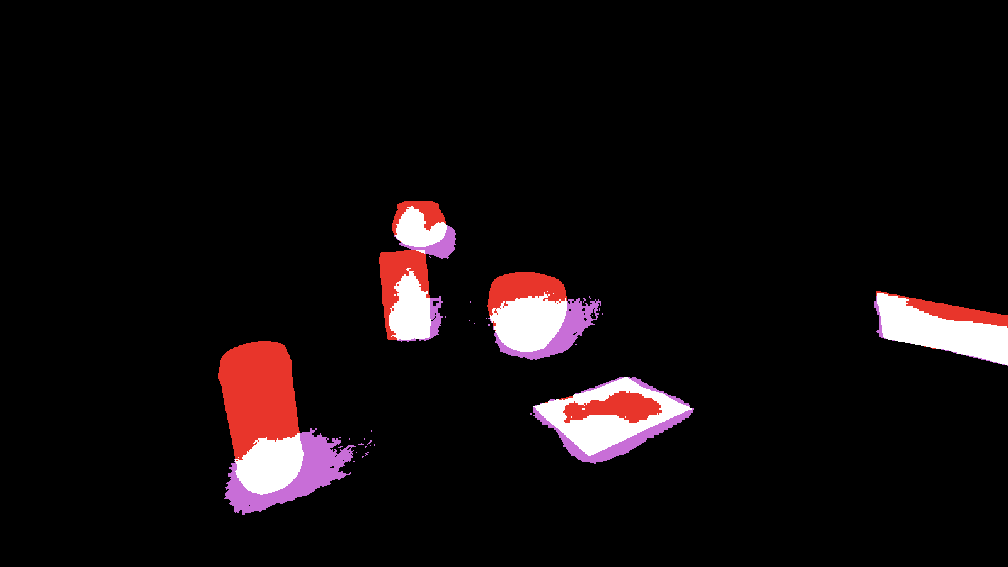}{0.440}{0.611}{0.764}{0.866}{1.2pt}
\qualrow{Garden}{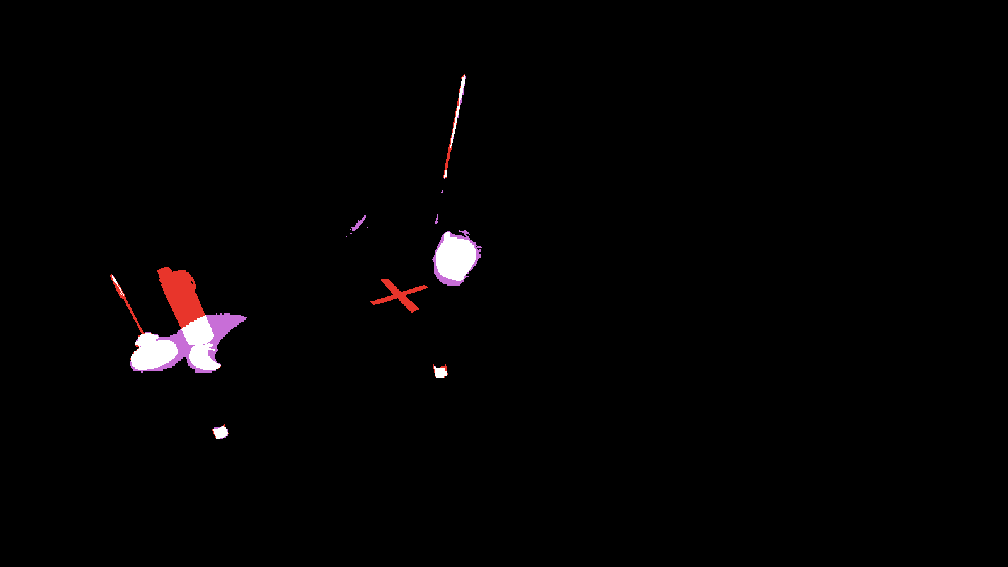}{0.472}{0.641}{0.739}{0.850}{1.2pt}
\qualrow{Lounge}{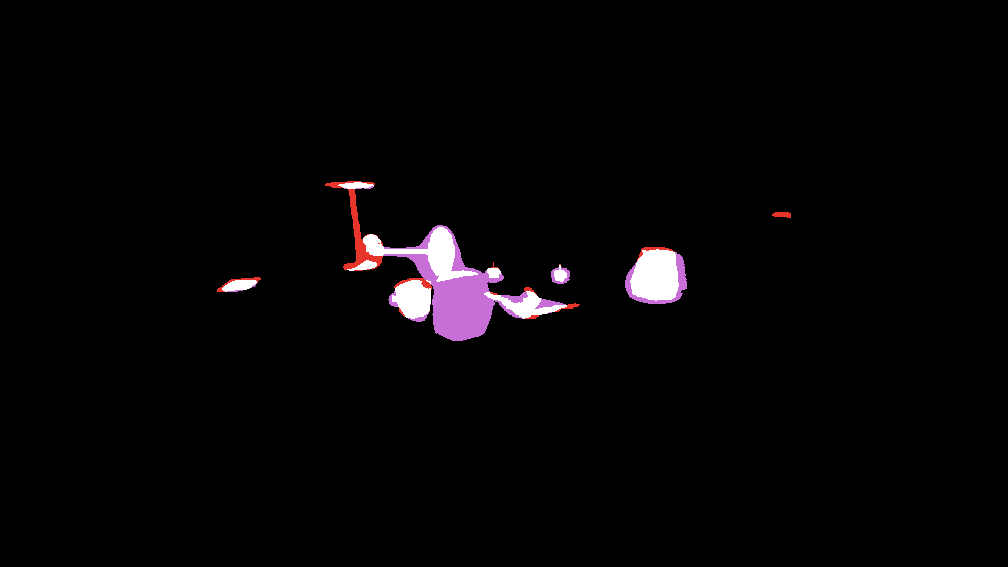}{0.478}{0.647}{0.744}{0.853}{1.2pt}
\qualrow{Meeting room}{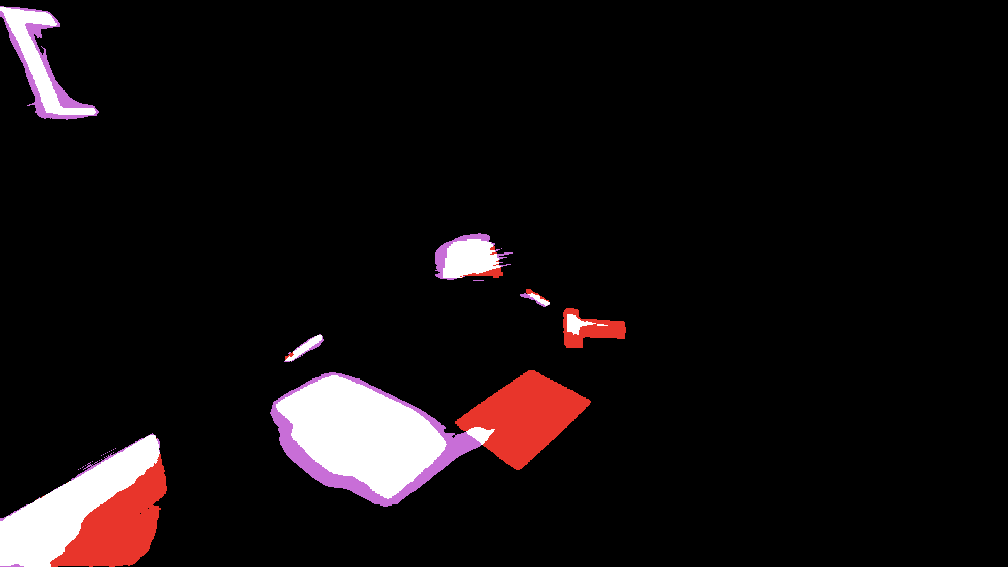}{0.532}{0.695}{0.710}{0.831}{1.2pt}
\qualrow{Lunch room}{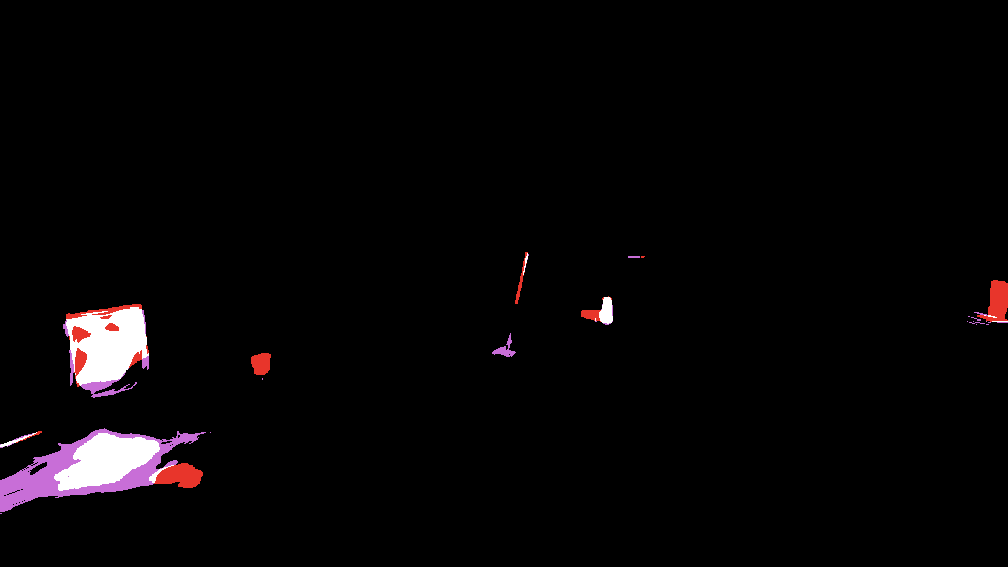}{0.485}{0.653}{0.662}{0.796}{0pt}
\end{tabular}

\vspace{4pt}
{\footnotesize
\qualkey{white}{True positive}\qquad
\qualkey{black}{True negative}\qquad
\qualkey{qualfn}{False negative}\qquad
\qualkey{qualfp}{False positive}}

\caption{Qualitative comparison with O\protect\nobreakdash-SCD~\cite{galappaththige2025changes} on PASLCD~\cite{galappaththige2025multi}. Each row shows the reference map rendered at the estimated pose of the inference image, the post-change inference image, and the change masks of O\protect\nobreakdash-SCD and O\protect\nobreakdash-SCD++, colour-coded against the ground truth, with per-frame IoU and F1. O\protect\nobreakdash-SCD++ also recovers particularly challenging surface-level changes that barely contrast with their surroundings: a grey cross on a grey chair (Garden), a white laminated sheet on a white table (Meeting room), and a water spill, transparent cup on a table (Lunch room).}
\label{fig:qualitative}
\end{figure*}
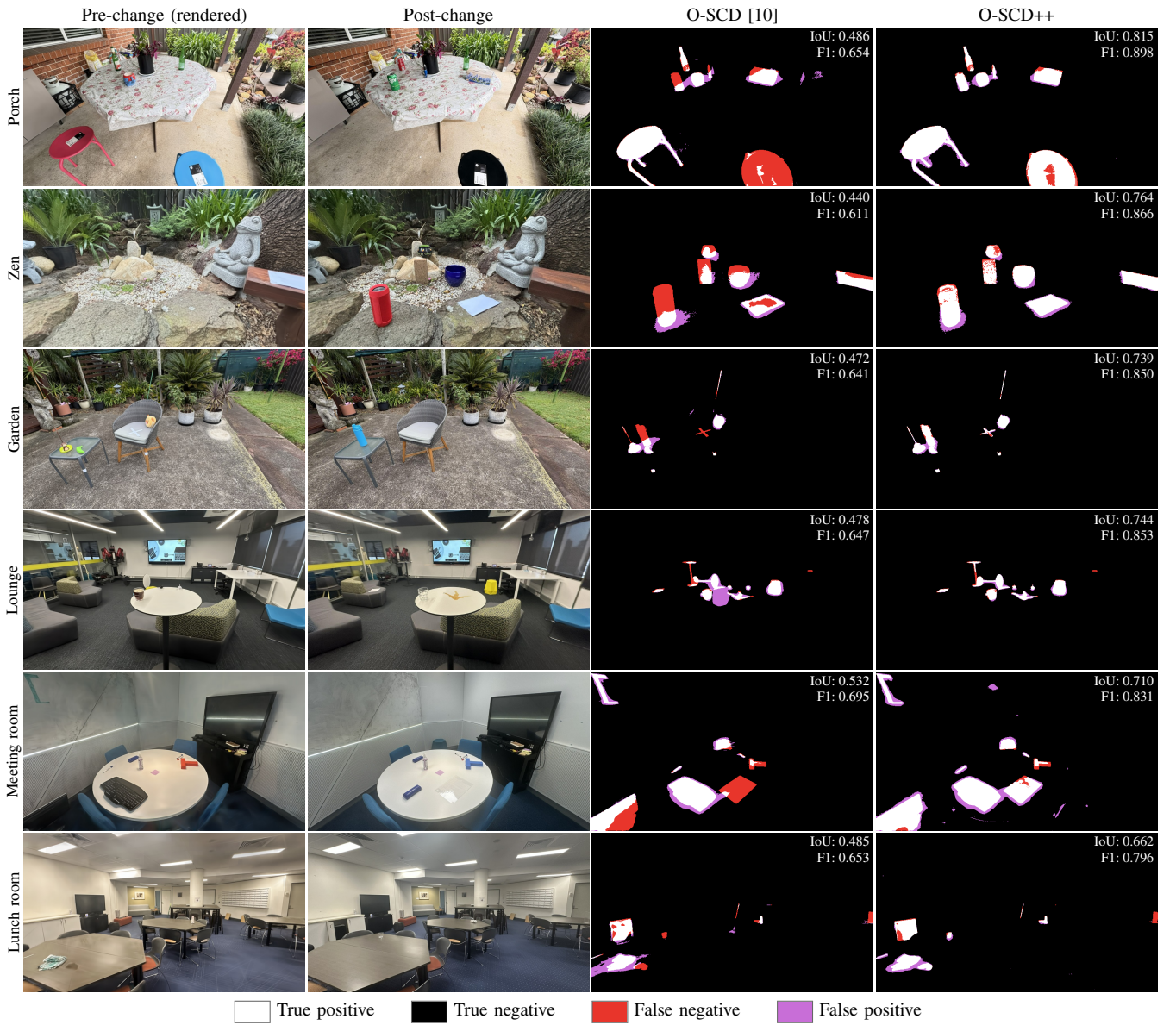

\begin{table}[t]
\centering
\footnotesize
\setlength{\tabcolsep}{2pt}
\caption{SCD results on PASLCD~\cite{galappaththige2025multi}, averaged over all 20 instances. LF: label-free; PA: pose-agnostic; MV: multi-view. The offline O\protect\nobreakdash-SCD result is included for reference. O\protect\nobreakdash-SCD++ achieves the state of the art performance in both accuracy and throughput. }
\label{tab:paslcd}
\begin{tabular*}{0.9\columnwidth}{@{\extracolsep{\fill}}lcccccc}
\toprule
Method & LF & PA & MV & mIoU & F1 & FPS \\
\midrule
O\nobreakdash-SCD (offline)~\cite{galappaththige2025changes} & \checkmark & \checkmark & \checkmark & 0.556 & 0.700 & -- \\
\midrule
ChangeSim (CS)~\cite{park2021changesim} & -- & -- & -- & 0.018 & 0.034 & 11.5 \\
CS+CYWS2D~\cite{cyws2d} & -- & -- & -- & 0.243 & 0.360 & 8.2 \\
CS+GeSCF~\cite{kim2025towards} & \checkmark & -- & -- & 0.181 & 0.270 & $<$1 \\
FastGeSCF~\cite{yeh2026towards} & -- & -- & -- & 0.122 & 0.198 & 1.43 \\
OmniposeAD~\cite{zhou2023pad} & \checkmark & \checkmark & -- & 0.168 & 0.262 & $<$1 \\
SplatPose~\cite{kruse2024splatpose} & \checkmark & \checkmark & -- & 0.173 & 0.281 & $<$1 \\
SplatPose+~\cite{liu2024splatpose+} & \checkmark & \checkmark & -- & 0.237 & 0.358 & $<$1 \\
O\nobreakdash-SCD~\cite{galappaththige2025changes} & \checkmark & \checkmark & \checkmark & 0.489 & 0.643 & 11.8 \\
O\nobreakdash-SCD++ & \checkmark & \checkmark & \checkmark & \textbf{0.575} & \textbf{0.715} & \textbf{25.6} \\
\bottomrule
\end{tabular*}
\end{table}

\Cref{tab:paslcd} compares O\nobreakdash-SCD++ with online SCD methods on PASLCD~\cite{galappaththige2025multi}. O\nobreakdash-SCD++ achieves the highest mIoU and F1 of all methods and the highest throughput, \emph{more than twice that of any existing approach}. \cref{fig:qualitative} shows a qualitative comparison.

Methods that retain the pairwise formulation perform poorly on these complex scenes. Combining the frame retrieval of ChangeSim with the strongest pairwise detectors reaches at most $0.243$ mIoU. The pose-agnostic AD methods remove the dependence on aligned trajectories but still infer change independently for each frame, reaching at most $0.237$ mIoU at below $1$\,FPS. Neither family exploits evidence across views, and the two multi-view methods, O\nobreakdash-SCD and O\nobreakdash-SCD++, lead both by a wide margin.

Supervision is a second factor. ChangeSim, whose change detector is trained on annotated change data, and FastGeSCF, whose pseudo-mask generator is, both perform below every label-free method; in each case the annotations come from a specific domain. The lone exception among supervised methods is CYWS2D. It is pretrained at large scale on pairs of real COCO images~\cite{lin2014microsoft} in which the changes are synthesized by inpainting or pasting objects, and an independent random affine transformation is applied to each image of a pair~\cite{cyws2d}. Its training thus combines large-scale real imagery with geometric misalignment between the two images. \Cref{sec:analysis} shows that its advantage is largely confined to outdoor scenes.

Compared with the conference version, O\nobreakdash-SCD, O\nobreakdash-SCD++ raises mIoU from $0.489$ to $0.575$, an $18\%$ relative improvement, while increasing throughput $2.2\times$, from $11.8$ to $25.6$\,FPS. It also exceeds the offline result of O\nobreakdash-SCD ($0.556$ mIoU, $0.700$ F1), obtained by observing the whole sequence before inferring any mask, while remaining causal.

\begin{table}[t]\centering
\footnotesize
\setlength{\tabcolsep}{2pt}
\caption{SCD results on CL-Splats~\cite{ackermann2025clsplats} and ChangeSim~\cite{park2021changesim}, averaged over scenes. O\protect\nobreakdash-SCD++ achieves the best change detection accuracy on both.}
\label{tab:clsplats_changesim}
\begin{tabular*}{0.9\columnwidth}{@{\extracolsep{\fill}}lcccc}
\toprule
 & \multicolumn{2}{c}{CL-Splats} & \multicolumn{2}{c}{ChangeSim} \\
\cmidrule(lr){2-3}\cmidrule(lr){4-5}
Method & mIoU & F1 & mIoU & F1 \\
\midrule
ChangeSim (CS)~\cite{park2021changesim} & 0.010 & 0.019 & 0.130 & 0.212 \\
CS+CYWS2D~\cite{cyws2d} & 0.372 & 0.519 & 0.271 & 0.408 \\
CS+GeSCF~\cite{kim2025towards} & 0.520 & 0.630 & 0.348 & 0.484 \\
FastGeSCF~\cite{yeh2026towards} & 0.183 & 0.252 & 0.315 & 0.458 \\
OmniposeAD~\cite{zhou2023pad} & 0.166 & 0.247 & 0.073 & 0.125 \\
SplatPose~\cite{kruse2024splatpose} & 0.224 & 0.320 & 0.061 & 0.108 \\
SplatPose+~\cite{liu2024splatpose+} & 0.556 & 0.700 & 0.115 & 0.186 \\
O\nobreakdash-SCD~\cite{galappaththige2025changes} & 0.657 & 0.773 & 0.404 & 0.535 \\
O\nobreakdash-SCD++ & \textbf{0.703} & \textbf{0.807} & \textbf{0.474} & \textbf{0.605} \\
\bottomrule
\end{tabular*}
\end{table}

The same ordering holds on CL-Splats and ChangeSim (\cref{tab:clsplats_changesim}), although the margin over methods without multi-view fusion narrows. Each CL-Splats scene contains a single, well-defined object change in a scene of small extent, close to the setting that pairwise detectors and pose-agnostic AD methods are designed for: CS+GeSCF reaches $0.520$ mIoU and SplatPose+ $0.556$, against at most $0.243$ on PASLCD. On ChangeSim, the revisit closely follows the reference trajectory, so the retrieved reference frame is nearly aligned with the inference image and the near-identical-viewpoint assumption of pairwise methods largely holds. Nevertheless, O\nobreakdash-SCD and O\nobreakdash-SCD++ surpass all baselines on both datasets, and O\nobreakdash-SCD++ achieves the highest mIoU and F1 on CL-Splats and ChangeSim. On ChangeSim, the absolute accuracy of both methods is limited by the pose front end rather than the change detection stage. XFeat~\cite{potje2024xfeat}, designed for real images, matches less reliably on the synthetic imagery, and some inference images fail to register. Each unregistered image receives an empty mask that scores zero IoU, which lowers the overall accuracy. Appendix~\ref{app:changesim} shows that a small modification to the pose front end, lifting the reference keypoints with depth rendered from the reference map instead of triangulating them, registers every image and raises O\nobreakdash-SCD++ to $0.558$ mIoU while preserving its margin over O\nobreakdash-SCD.

\subsection{Ablations}
\label{sec:ablation}

\begin{table}[t]
\centering
\footnotesize
\setlength{\tabcolsep}{2pt}
\caption{Ablations of O\protect\nobreakdash-SCD++ on PASLCD~\cite{galappaththige2025multi}. Each block varies one component while the others are fixed. The oracle selects, per frame, the threshold that maximizes IoU against the ground truth and serves as an upper bound.}
\label{tab:ablation}
\begin{tabular*}{0.9\columnwidth}{@{\extracolsep{\fill}}lcc}
\toprule
Variant & mIoU & F1 \\
\midrule
\multicolumn{3}{l}{\emph{Current-view prediction (\cref{sec:prediction})}} \\
Cue only (no fusion) & 0.454 & 0.595 \\
Map only & 0.526 & 0.676 \\
Map and cue (ours) & \textbf{0.575} & \textbf{0.715} \\
\midrule
\multicolumn{3}{l}{\emph{Aggregation weight (\cref{sec:aggregation})}} \\
Unbounded mass, $\rho_{g,t}=m_{g,t}$ & 0.562 & 0.701 \\
Bounded weight, \cref{eq:rho} (ours) & \textbf{0.575} & \textbf{0.715} \\
\midrule
\multicolumn{3}{l}{\emph{Decision rule (\cref{sec:decision})}} \\
Fixed cutoff $0.5$~\cite{galappaththige2025changes} & 0.522 & 0.671 \\
Otsu~\cite{otsu1979threshold} & 0.313 & 0.422 \\
Exact scan (ours) & \textbf{0.575} & \textbf{0.715} \\
Oracle, per frame (not label-free) & \emph{0.640} & \emph{0.770} \\
\bottomrule
\end{tabular*}
\end{table}

\noindent\textbf{Current-view prediction.} \Cref{tab:ablation} shows that neither term of \cref{eq:score}, the rendered fused state $\Amat\mathbf{q}_t$ nor the coverage-scaled current cue $a_t c_t$, suffices alone. Deciding from the current cue alone, that is, without fusion, reaches $0.454$ mIoU, and the fused map alone reaches $0.526$; averaging the two gives $0.575$.

\noindent\textbf{Aggregation weight.} Weighting each view by the unbounded observation mass instead of the bounded weight of \cref{eq:rho} lowers mIoU from $0.575$ to $0.562$, as the views in which a primitive covers the most pixels then dominate its estimate.

\noindent\textbf{Decision rule.} We compare the exact decision of \cref{sec:decision} with two label-free alternatives applied to the same scores: the fixed cutoff of $0.5$ used by O\nobreakdash-SCD, and Otsu's method~\cite{otsu1979threshold}. Ours improves over both. The fixed cutoff also separates the two sources of the improvement over O\nobreakdash-SCD. With the cutoff of $0.5$, O\nobreakdash-SCD++ uses the same decision rule as O\nobreakdash-SCD, so the gain from $0.489$ (\cref{tab:paslcd}) to $0.522$ mIoU is due to the replay-free fusion of \cref{sec:transfer,sec:aggregation,sec:prediction} alone; the exact decision adds the remaining gain to $0.575$. The exact decision is specific to O\nobreakdash-SCD++: O\nobreakdash-SCD already minimizes the same objective by descent during fusion, which splits its field at the midpoint, so its fixed cutoff only reads off a decision that the objective has already made. Otsu's method assumes a bimodal score distribution and fails because changed pixels occupy a small fraction of each image. A per-frame oracle threshold chosen with ground truth reaches $0.640$ mIoU; the gap is the cost of deciding without labels.

\subsection{Analysis}
\label{sec:analysis}

\noindent\textbf{Indoor and outdoor scenes.} \Cref{tab:indoor_outdoor} splits PASLCD into its indoor and outdoor scenes, comparing O\nobreakdash-SCD++ with O\nobreakdash-SCD and the two strongest online baselines of \cref{tab:paslcd}, CS+CYWS2D and SplatPose+. O\nobreakdash-SCD++ performs almost identically in both settings ($0.573$ and $0.576$ mIoU), as does O\nobreakdash-SCD. CS+CYWS2D, in contrast, nearly halves from $0.319$ mIoU outdoors to $0.167$ indoors. An affine transformation approximates the difference between two viewpoints well only when scene depth is large relative to the camera displacement, as in outdoor scenes; the close-range geometry of indoor scenes instead produces strong parallax between an inference image and its retrieved reference frame, which is likely outside the affine misalignment that CYWS2D is trained on. SplatPose+ also degrades indoors, though less ($0.260$ to $0.213$).

\noindent\textbf{Lighting.} \Cref{tab:lighting} separates the instances revisited under lighting similar to the reference capture from those revisited under different lighting. O\nobreakdash-SCD++ is the most accurate method under both conditions, reaching $0.602$ and $0.548$ mIoU, and exceeds offline O\nobreakdash-SCD in each. Under different lighting, O\nobreakdash-SCD++ even surpasses O\nobreakdash-SCD under similar lighting ($0.548$ against $0.506$ mIoU). Different lighting lowers its mIoU by $0.054$, compared with $0.034$ for O\nobreakdash-SCD, while SplatPose+ is essentially unaffected at a much lower level of accuracy.

\begin{table}[t]
\centering
\footnotesize
\setlength{\tabcolsep}{2pt}
\caption{Indoor and outdoor scenes of PASLCD~\cite{galappaththige2025multi} for the two strongest online baselines of \cref{tab:paslcd}. O\protect\nobreakdash-SCD++ along with O\protect\nobreakdash-SCD perform similarly in both indoor and outdoor scenes.}
\label{tab:indoor_outdoor}
\begin{tabular*}{0.9\columnwidth}{@{\extracolsep{\fill}}lcccc}
\toprule
 & \multicolumn{2}{c}{Indoor} & \multicolumn{2}{c}{Outdoor} \\
\cmidrule(lr){2-3}\cmidrule(lr){4-5}
Method & mIoU & F1 & mIoU & F1 \\
\midrule
O\nobreakdash-SCD (offline)~\cite{galappaththige2025changes} & 0.555 & 0.706 & 0.557 & 0.694 \\
\midrule
CS+CYWS2D~\cite{cyws2d} & 0.167 & 0.272 & 0.319 & 0.448 \\
SplatPose+~\cite{liu2024splatpose+} & 0.213 & 0.326 & 0.260 & 0.390 \\
O\nobreakdash-SCD~\cite{galappaththige2025changes} & 0.484 & 0.641 & 0.494 & 0.645 \\
O\nobreakdash-SCD++ & \textbf{0.573} & \textbf{0.718} & \textbf{0.576} & \textbf{0.713} \\
\bottomrule
\end{tabular*}
\end{table}

\begin{table}[t]
\centering
\footnotesize
\setlength{\tabcolsep}{2pt}
\caption{Results under lighting similar to, and different from, the reference capture, for the two strongest online baselines of \cref{tab:paslcd} on PASLCD~\cite{galappaththige2025multi}. O\protect\nobreakdash-SCD++ achieves the best accuracy under both lighting conditions.}
\label{tab:lighting}
\begin{tabular*}{0.9\columnwidth}{@{\extracolsep{\fill}}lcccc}
\toprule
 & \multicolumn{2}{c}{Similar lighting} & \multicolumn{2}{c}{Different lighting} \\
\cmidrule(lr){2-3}\cmidrule(lr){4-5}
Method & mIoU & F1 & mIoU & F1 \\
\midrule
O\nobreakdash-SCD (offline)~\cite{galappaththige2025changes} & 0.580 & 0.718 & 0.532 & 0.681 \\
\midrule
CS+CYWS2D~\cite{cyws2d} & 0.258 & 0.378 & 0.229 & 0.342 \\
SplatPose+~\cite{liu2024splatpose+} & 0.236 & 0.357 & 0.237 & 0.359 \\
O\nobreakdash-SCD~\cite{galappaththige2025changes} & 0.506 & 0.656 & 0.472 & 0.630 \\
O\nobreakdash-SCD++ & \textbf{0.602} & \textbf{0.736} & \textbf{0.548} & \textbf{0.694} \\
\bottomrule
\end{tabular*}
\end{table}

\begin{table}[t]
\centering
\footnotesize
\setlength{\tabcolsep}{2pt}
\caption{Online, retrospective and offline inference on PASLCD~\cite{galappaththige2025multi}. Retrospective inference re-infers every frame from the state reached after the last frame; offline O\protect\nobreakdash-SCD instead continues to refine its change representation up to $3000$ iterations. }
\label{tab:retrospective}
\begin{tabular*}{0.9\columnwidth}{@{\extracolsep{\fill}}lcccc}
\toprule
Method & Setting & Extra optimization & mIoU & F1 \\
\midrule
O\nobreakdash-SCD~\cite{galappaththige2025changes} & online & -- & 0.489 & 0.643 \\
O\nobreakdash-SCD~\cite{galappaththige2025changes} & retrospective & none & 0.511 & 0.662 \\
O\nobreakdash-SCD~\cite{galappaththige2025changes} & offline & refinement & 0.556 & 0.700 \\
\midrule
O\nobreakdash-SCD++ & online & -- & 0.575 & 0.715 \\
O\nobreakdash-SCD++ & retrospective & none & \textbf{0.586} & \textbf{0.726} \\
\bottomrule
\end{tabular*}
\end{table}

\noindent\textbf{Retrospective inference.} An online method emits each mask when its frame arrives, from the evidence observed up to that frame. Once the sequence ends, every earlier frame can be inferred again from the final state; we call this retrospective inference. For O\nobreakdash-SCD++, the final state is the exact aggregate over all frames (\cref{prop:exact}), so retrospective inference already yields its offline result, without further optimization. O\nobreakdash-SCD reaches its offline result only by continuing to refine its change representation up to $3000$ iterations~\cite{galappaththige2025changes}. \Cref{tab:retrospective} compares the three settings. Retrospective inference improves both methods. The gap between online and offline inference narrows from $0.067$ mIoU for O\nobreakdash-SCD to $0.011$ for O\nobreakdash-SCD++. Online O\nobreakdash-SCD++ already exceeds offline O\nobreakdash-SCD ($0.556$ mIoU), and retrospective inference extends its lead to $0.030$ mIoU without any optimization. \Cref{fig:retrospective} shows examples: the final state suppresses spurious detections that the evidence observed online could not yet rule out.

\begin{figure*}[t]
\centering
\definecolor{qualfn}{RGB}{232,53,43}
\definecolor{qualfp}{RGB}{200,110,215}
\newcommand{\retrokey}[2]{\fcolorbox{black!60}{#1}{\rule{0pt}{3pt}\rule{9pt}{0pt}}~#2}
\newcommand{\retrotile}[1]{\includegraphics[width=0.238\textwidth]{figures/retrospective/#1}}
\newcommand{\retromask}[4]{%
\begin{tikzpicture}
\node[inner sep=0pt, anchor=south west] (img) at (0,0) {\retrotile{#1}};
\node[anchor=#4, inner sep=1.5pt, text=white, fill=black, fill opacity=0.6, text opacity=1, font=\fontsize{6}{7}\selectfont, align=left]
    at (img.#4) {IoU: #2\\F1: #3};
\end{tikzpicture}}
\newcommand{\retrorow}[8]{%
\rotatebox{90}{\parbox{0.134\textwidth}{\centering\scriptsize #1}} &
\retrotile{#2_ref.jpg} & \retrotile{#2_inf.jpg} &
\retromask{#2_online.png}{#3}{#4}{#8} & \retromask{#2_retro.png}{#5}{#6}{#8} \\[#7]}
\setlength{\tabcolsep}{0.6pt}
\renewcommand{\arraystretch}{0}
\begin{tabular}{ccccc}
 & \footnotesize Pre-change (rendered) & \footnotesize Post-change & \footnotesize O-SCD++ (online) & \footnotesize O-SCD++ (retrospective) \\[3pt]
\retrorow{Pots}{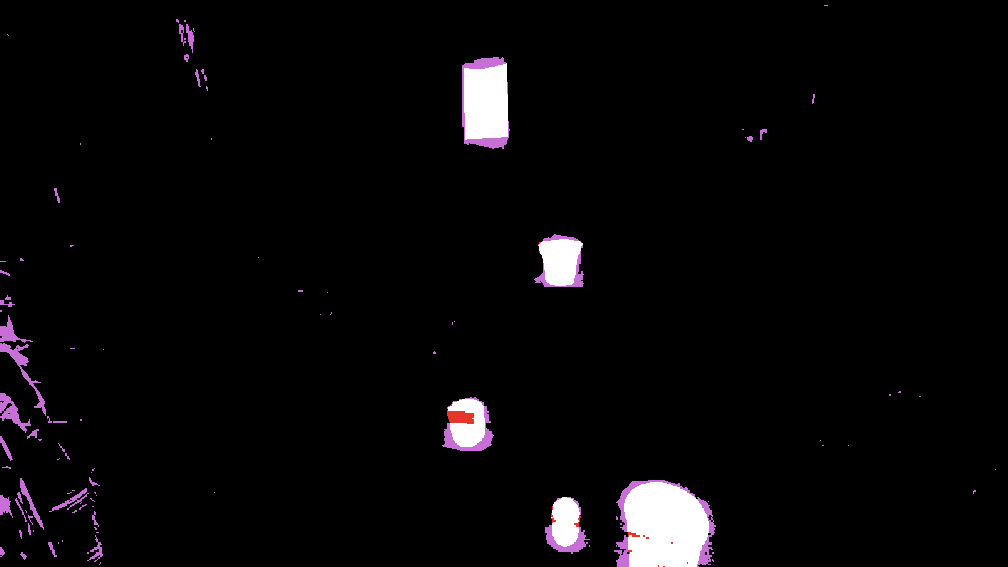}{0.644}{0.783}{0.741}{0.851}{1.2pt}{north east}
\retrorow{Lunch room}{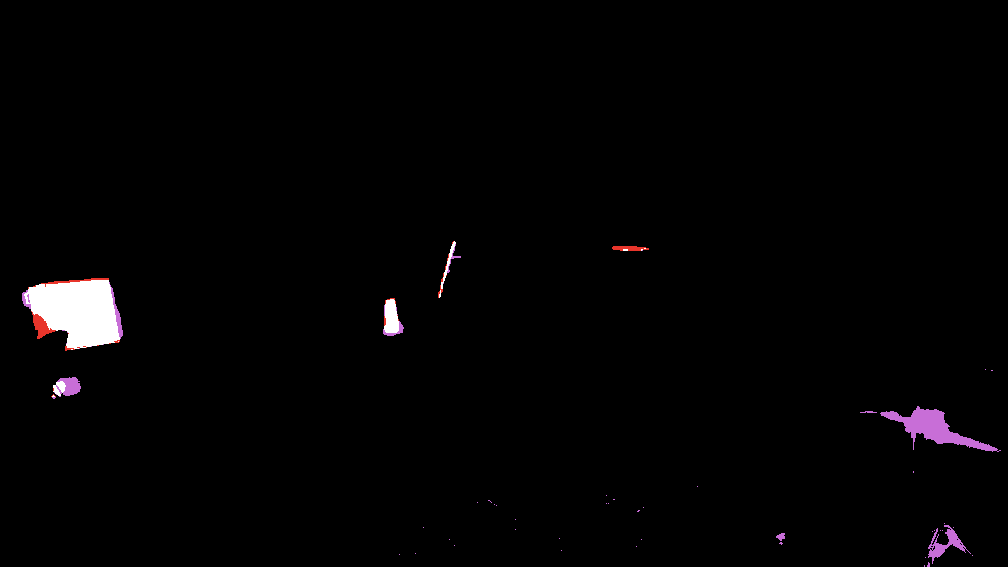}{0.582}{0.736}{0.673}{0.805}{0pt}{north east}
\end{tabular}

\vspace{4pt}
{\footnotesize
\retrokey{white}{True positive}\qquad
\retrokey{black}{True negative}\qquad
\retrokey{qualfn}{False negative}\qquad
\retrokey{qualfp}{False positive}}

\caption{Retrospective inference with O\protect\nobreakdash-SCD++ on PASLCD~\cite{galappaththige2025multi}. Online masks are emitted when each frame arrives; retrospective masks re-infer the same frames from the state reached after the last frame. The final state suppresses spurious detections that the evidence available online could not yet rule out.}
\label{fig:retrospective}
\end{figure*}
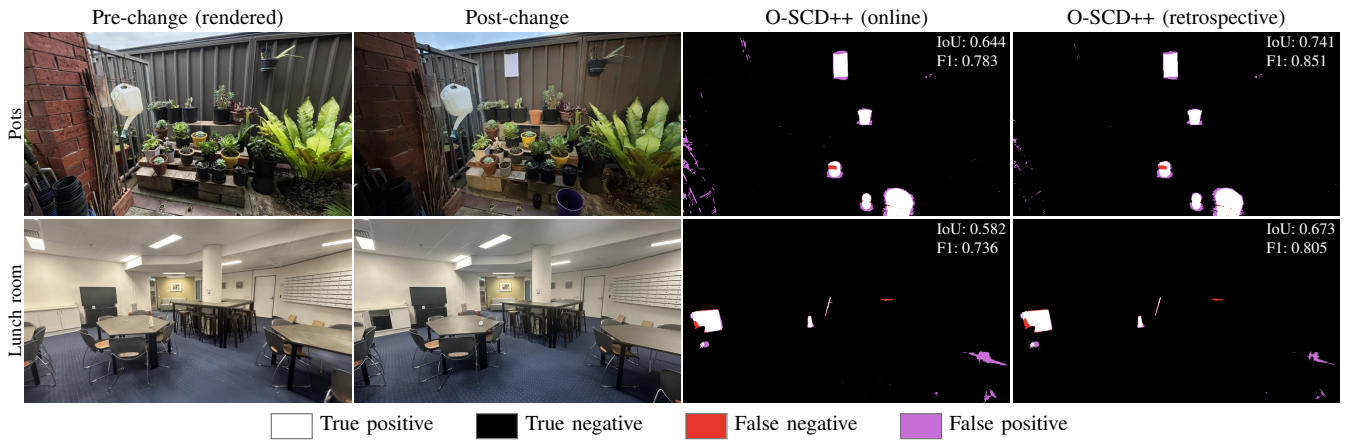

\begin{figure*}[t]
\centering
\definecolor{qualfn}{RGB}{232,53,43}
\definecolor{qualfp}{RGB}{200,110,215}
\newcommand{\unckey}[2]{\fcolorbox{black!60}{#1}{\rule{0pt}{3pt}\rule{9pt}{0pt}}~#2}
\newcommand{\unctile}[1]{\includegraphics[width=0.19\textwidth]{figures/uncertainty/#1}}
\newcommand{\uncmask}[4]{%
\begin{tikzpicture}
\node[inner sep=0pt, anchor=south west] (img) at (0,0) {\unctile{#1}};
\node[anchor=#4, inner sep=1.5pt, text=white, fill=black, fill opacity=0.6, text opacity=1, font=\fontsize{6}{7}\selectfont, align=left]
    at (img.#4) {IoU: #2\\F1: #3};
\end{tikzpicture}}
\newcommand{\uncrow}[8]{%
\rotatebox{90}{\parbox{0.107\textwidth}{\centering\scriptsize #1}} &
\unctile{#2_ref.jpg} & \unctile{#2_inf.jpg} & \unctile{#2_unc.jpg} &
\uncmask{#2_wo.png}{#3}{#4}{#8} & \uncmask{#2_with.png}{#5}{#6}{#8} \\[#7]}
\setlength{\tabcolsep}{0.6pt}
\renewcommand{\arraystretch}{0}
\begin{tabular}{cccccc}
 & \footnotesize Pre-change (rendered) & \footnotesize Post-change & \footnotesize Photometric uncertainty & \footnotesize O-SCD++ & \footnotesize O-SCD++ + uncertainty \\[3pt]
\uncrow{Zen}{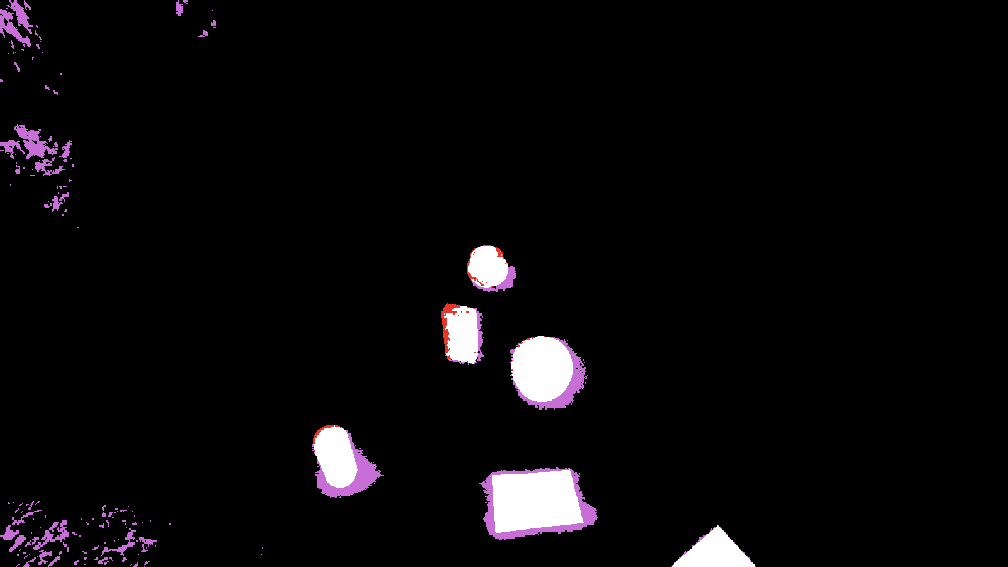}{0.513}{0.678}{0.608}{0.756}{1.2pt}{north east}
\uncrow{Meeting room}{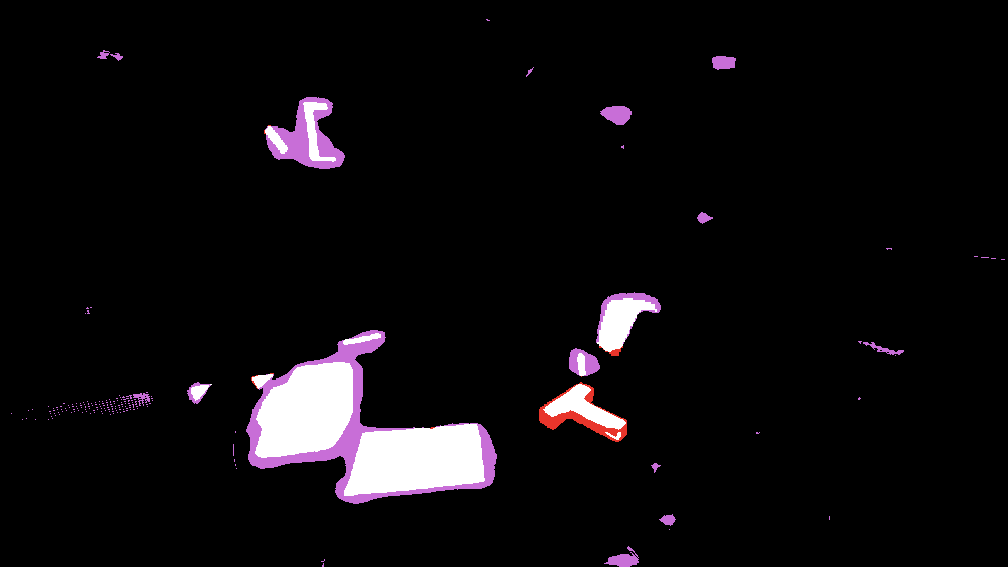}{0.556}{0.715}{0.609}{0.757}{0pt}{north east}
\end{tabular}

\vspace{4pt}
{\footnotesize
\unckey{white}{True positive}\qquad
\unckey{black}{True negative}\qquad
\unckey{qualfn}{False negative}\qquad
\unckey{qualfp}{False positive}}

\caption{O\protect\nobreakdash-SCD++ with and without predictive photometric uncertainty~\cite{galappaththige2026predictive} on PASLCD~\cite{galappaththige2025multi}. The rendered uncertainty (blue: low, red: high) is high where the reference map is poorly reconstructed; attenuating the cue there removes false positives caused by rendering artifacts~\cite{galappaththige2026pixels}.}
\label{fig:uncertainty}
\end{figure*}
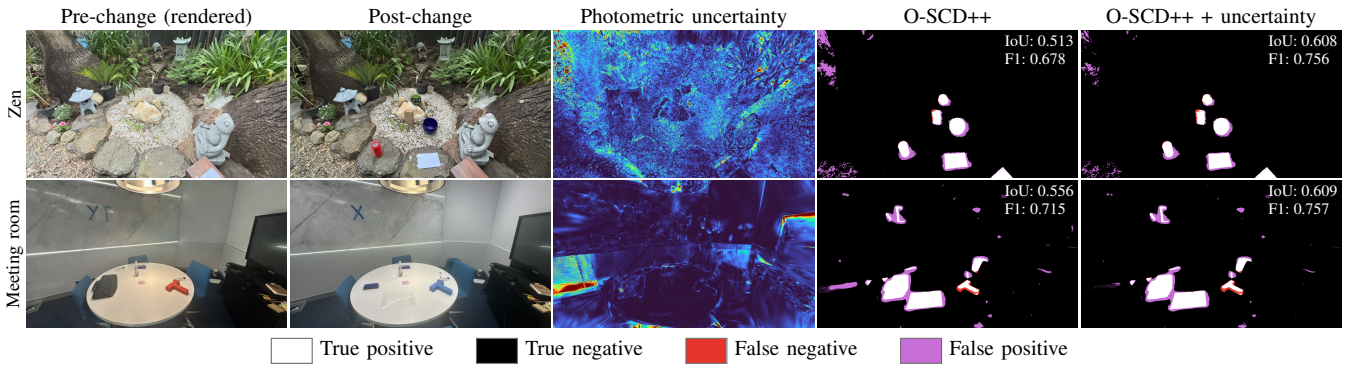

\begin{table}[t]
\centering
\footnotesize
\setlength{\tabcolsep}{2pt}
\caption{Effect of predictive photometric uncertainty~\cite{galappaththige2026predictive} on PASLCD~\cite{galappaththige2025multi}, CL-Splats~\cite{ackermann2025clsplats} and ChangeSim~\cite{park2021changesim}. The rendered uncertainty attenuates the change cue where the reference scene rendering is unreliable, improving the change detection accuracy.}
\label{tab:uncertainty}
\begin{tabular*}{0.9\columnwidth}{@{\extracolsep{\fill}}llcccc}
\toprule
 & & \multicolumn{2}{c}{Without} & \multicolumn{2}{c}{With uncertainty} \\
\cmidrule(lr){3-4}\cmidrule(lr){5-6}
Dataset & Method & mIoU & F1 & mIoU & F1 \\
\midrule
PASLCD & O\nobreakdash-SCD~\cite{galappaththige2025changes} & 0.489 & 0.643 & 0.493 & 0.648 \\
 & O\nobreakdash-SCD++ & 0.575 & 0.715 & \textbf{0.602} & \textbf{0.741} \\
\midrule
CL-Splats & O\nobreakdash-SCD~\cite{galappaththige2025changes} & 0.657 & 0.773 & 0.688 & 0.796 \\
 & O\nobreakdash-SCD++ & 0.703 & 0.807 & \textbf{0.738} & \textbf{0.831} \\
\midrule
ChangeSim & O\nobreakdash-SCD~\cite{galappaththige2025changes} & 0.404 & 0.535 & 0.419 & 0.553 \\
 & O\nobreakdash-SCD++ & 0.474 & 0.605 & \textbf{0.486} & \textbf{0.615} \\
\bottomrule
\end{tabular*}
\end{table}

\begin{table}[t]
\centering
\footnotesize
\setlength{\tabcolsep}{2pt}
\caption{Per-frame processing time of each module on PASLCD~\cite{galappaththige2025multi}, in milliseconds and as a share of the frame, measured on an NVIDIA RTX~4090 at a resolution of about $1008\times567$ (scene-dependent).}
\label{tab:runtime}
\begin{tabular*}{0.9\columnwidth}{@{\extracolsep{\fill}}lcccc}
\toprule
 & \multicolumn{2}{c}{O\nobreakdash-SCD~\cite{galappaththige2025changes}} & \multicolumn{2}{c}{O\nobreakdash-SCD++} \\
\cmidrule(lr){2-3}\cmidrule(lr){4-5}
Module & ms & \% & ms & \% \\
\midrule
\textbf{Pose estimation} & \textbf{30.4} & \textbf{35.8\%} & \textbf{30.2} & \textbf{77.0\%} \\
\midrule
\textbf{Change detection} & \textbf{54.4} & \textbf{64.2\%} & \textbf{9.0} & \textbf{23.0\%} \\
\quad Change cue generation & 6.1 & 7.2\% & 5.9 & 15.0\% \\
\quad Multi-view fusion & 47.8 & 56.4\% & 2.1 & 5.4\% \\
\quad Mask inference & 0.5 & 0.6\% & 1.0 & 2.6\% \\
\bottomrule
\end{tabular*}
\end{table}

\begin{figure*}[t]
\centering
\definecolor{spred}{RGB}{214,39,40}
\definecolor{spblue}{RGB}{31,119,180}
\begin{tikzpicture}
\begin{axis}[
    width=8.6cm, height=5.2cm,
    title={O-SCD~\cite{galappaththige2025changes}}, title style={font=\footnotesize, yshift=-1ex},
    xlabel={Fraction of pixels withheld}, ylabel={mIoU on retained pixels},
    xmin=0, xmax=0.5, ymin=0, ymax=1.02,
    xtick={0,0.1,0.2,0.3,0.4,0.5}, ytick={0,0.2,0.4,0.6,0.8,1.0},
    grid=major, grid style={line width=0.3pt, draw=gray!30},
    tick label style={font=\footnotesize}, label style={font=\footnotesize},
    legend style={font=\fontsize{6}{7}\selectfont, at={{(0.97,0.525)}}, anchor=south east, draw=black!50, fill=white, fill opacity=0.9, text opacity=1, inner sep=1pt, row sep=0pt},
    legend cell align=left,
]
\addplot[color=black, dashed, line width=0.9pt] coordinates {(0.00,0.4893) (0.02,0.8874) (0.04,0.9596) (0.06,0.9815) (0.08,0.9891) (0.10,0.9912) (0.12,0.9940) (0.14,0.9947) (0.16,0.9952) (0.18,0.9962) (0.20,0.9962) (0.22,0.9962) (0.24,0.9962) (0.26,0.9963) (0.28,0.9964) (0.30,0.9965) (0.32,0.9967) (0.34,0.9978) (0.36,0.9987) (0.38,0.9998) (0.40,1.0000) (0.42,1.0000) (0.44,1.0000) (0.46,1.0000) (0.48,1.0000) (0.50,1.0000)};
\addlegendentry{oracle (errors first)}
\addplot[color=spred, line width=1.1pt] coordinates {(0.00,0.4893) (0.02,0.5116) (0.04,0.4566) (0.06,0.3918) (0.08,0.3472) (0.10,0.3031) (0.12,0.2573) (0.14,0.2234) (0.16,0.1927) (0.18,0.1678) (0.20,0.1448) (0.22,0.1237) (0.24,0.1103) (0.26,0.1058) (0.28,0.1020) (0.30,0.0973) (0.32,0.0935) (0.34,0.0859) (0.36,0.0782) (0.38,0.0723) (0.40,0.0667) (0.42,0.0646) (0.44,0.0639) (0.46,0.0647) (0.48,0.0693) (0.50,0.0688)};
\addlegendentry{distance to fixed cutoff $0.5$ {\fontsize{5}{6}\selectfont(AUSE $0.801$)}}
\addplot[color=spblue, line width=1.1pt] coordinates {(0.00,0.4893) (0.02,0.4905) (0.04,0.4916) (0.06,0.4928) (0.08,0.4936) (0.10,0.4943) (0.12,0.4950) (0.14,0.4957) (0.16,0.4964) (0.18,0.4972) (0.20,0.4981) (0.22,0.4988) (0.24,0.4996) (0.26,0.5003) (0.28,0.5011) (0.30,0.5020) (0.32,0.5028) (0.34,0.5038) (0.36,0.5047) (0.38,0.5057) (0.40,0.5067) (0.42,0.5079) (0.44,0.5092) (0.46,0.5107) (0.48,0.5121) (0.50,0.5133)};
\addlegendentry{photometric uncertainty {\fontsize{5}{6}\selectfont(AUSE $0.480$)}}
\addplot[color=black!60, dotted, line width=0.9pt] coordinates {(0.00,0.4893) (0.02,0.4893) (0.04,0.4892) (0.06,0.4892) (0.08,0.4892) (0.10,0.4892) (0.12,0.4892) (0.14,0.4892) (0.16,0.4893) (0.18,0.4892) (0.20,0.4892) (0.22,0.4891) (0.24,0.4891) (0.26,0.4891) (0.28,0.4892) (0.30,0.4892) (0.32,0.4892) (0.34,0.4891) (0.36,0.4891) (0.38,0.4893) (0.40,0.4892) (0.42,0.4893) (0.44,0.4893) (0.46,0.4892) (0.48,0.4893) (0.50,0.4892)};
\addlegendentry{random {\fontsize{5}{6}\selectfont(AUSE $0.491$)}}
\end{axis}
\end{tikzpicture}
\hfill
\begin{tikzpicture}
\begin{axis}[
    width=8.6cm, height=5.2cm,
    title={O-SCD++}, title style={font=\footnotesize, yshift=-1ex},
    xlabel={Fraction of pixels withheld}, ylabel={mIoU on retained pixels},
    xmin=0, xmax=0.5, ymin=0, ymax=1.02,
    xtick={0,0.1,0.2,0.3,0.4,0.5}, ytick={0,0.2,0.4,0.6,0.8,1.0},
    grid=major, grid style={line width=0.3pt, draw=gray!30},
    tick label style={font=\footnotesize}, label style={font=\footnotesize},
    legend style={font=\fontsize{6}{7}\selectfont, at={{(0.97,0.03)}}, anchor=south east, draw=black!50, fill=white, fill opacity=0.9, text opacity=1, inner sep=1pt, row sep=0pt},
    legend cell align=left,
]
\addplot[color=black, dashed, line width=0.9pt] coordinates {(0.00,0.5746) (0.02,0.9196) (0.04,0.9742) (0.06,0.9851) (0.08,0.9882) (0.10,0.9920) (0.12,0.9925) (0.14,0.9926) (0.16,0.9927) (0.18,0.9929) (0.20,0.9934) (0.22,0.9951) (0.24,0.9963) (0.26,0.9964) (0.28,0.9964) (0.30,0.9965) (0.32,0.9967) (0.34,0.9972) (0.36,0.9983) (0.38,0.9984) (0.40,0.9990) (0.42,1.0000) (0.44,1.0000) (0.46,1.0000) (0.48,1.0000) (0.50,1.0000)};
\addlegendentry{oracle (errors first)}
\addplot[color=spred, line width=1.1pt] coordinates {(0.00,0.5746) (0.02,0.6944) (0.04,0.7529) (0.06,0.7888) (0.08,0.8131) (0.10,0.8310) (0.12,0.8452) (0.14,0.8566) (0.16,0.8655) (0.18,0.8727) (0.20,0.8788) (0.22,0.8844) (0.24,0.8896) (0.26,0.8944) (0.28,0.8991) (0.30,0.9032) (0.32,0.9069) (0.34,0.9099) (0.36,0.9125) (0.38,0.9149) (0.40,0.9174) (0.42,0.9199) (0.44,0.9221) (0.46,0.9239) (0.48,0.9260) (0.50,0.9286)};
\addlegendentry{decision margin {\fontsize{5}{6}\selectfont(AUSE $0.116$)}}
\addplot[color=spblue, line width=1.1pt] coordinates {(0.00,0.5746) (0.02,0.5832) (0.04,0.5884) (0.06,0.5925) (0.08,0.5956) (0.10,0.5983) (0.12,0.6007) (0.14,0.6029) (0.16,0.6050) (0.18,0.6071) (0.20,0.6091) (0.22,0.6108) (0.24,0.6125) (0.26,0.6141) (0.28,0.6157) (0.30,0.6174) (0.32,0.6190) (0.34,0.6205) (0.36,0.6220) (0.38,0.6236) (0.40,0.6252) (0.42,0.6267) (0.44,0.6282) (0.46,0.6297) (0.48,0.6313) (0.50,0.6330)};
\addlegendentry{photometric uncertainty {\fontsize{5}{6}\selectfont(AUSE $0.372$)}}
\addplot[color=black!60, dotted, line width=0.9pt] coordinates {(0.00,0.5746) (0.02,0.5745) (0.04,0.5745) (0.06,0.5745) (0.08,0.5745) (0.10,0.5745) (0.12,0.5744) (0.14,0.5744) (0.16,0.5744) (0.18,0.5745) (0.20,0.5745) (0.22,0.5745) (0.24,0.5744) (0.26,0.5744) (0.28,0.5744) (0.30,0.5744) (0.32,0.5744) (0.34,0.5745) (0.36,0.5744) (0.38,0.5745) (0.40,0.5744) (0.42,0.5745) (0.44,0.5745) (0.46,0.5744) (0.48,0.5744) (0.50,0.5743)};
\addlegendentry{random {\fontsize{5}{6}\selectfont(AUSE $0.409$)}}
\end{axis}
\end{tikzpicture}
\vspace{-6pt}
\caption{Sparsification of the online masks of O\protect\nobreakdash-SCD (left) and O\protect\nobreakdash-SCD++ (right) on PASLCD~\cite{galappaththige2025multi}. In each frame, a growing fraction of pixels is withheld in the order given, and mIoU is computed on the pixels that remain. For O\protect\nobreakdash-SCD, pixels are ranked by their distance to its fixed cutoff of $0.5$; for O\protect\nobreakdash-SCD++, by the decision margin of \cref{eq:margin}. Smaller margins and larger uncertainties are withheld first. The area under the sparsification error (AUSE)~\cite{ilg2018uncertainty} is the area between a curve and the oracle, normalized by the plotted range, so that it equals the mean gap to the oracle; lower is better.}
\label{fig:sparsification}
\vspace{10pt}
\centering
\definecolor{qualfn}{RGB}{232,53,43}
\definecolor{qualfp}{RGB}{200,110,215}
\newcommand{\marginkey}[2]{\fcolorbox{black!60}{#1}{\rule{0pt}{3pt}\rule{9pt}{0pt}}~#2}
\newcommand{\margintile}[1]{\includegraphics[width=0.238\textwidth]{figures/margin/#1}}
\newcommand{\marginrow}[5]{%
\rotatebox{90}{\parbox{0.134\textwidth}{\centering\scriptsize #1}} &
\margintile{#2_ref.jpg} & \margintile{#2_inf.jpg} &
\margintile{#2_pred.png} & \margintile{#2_margin.jpg} \\[#5]}
\setlength{\tabcolsep}{0.6pt}
\renewcommand{\arraystretch}{0}
\begin{tabular}{ccccc}
 & \footnotesize Pre-change (rendered) & \footnotesize Post-change & \footnotesize O-SCD++ prediction & \footnotesize Decision margin \\[3pt]
\marginrow{Porch}{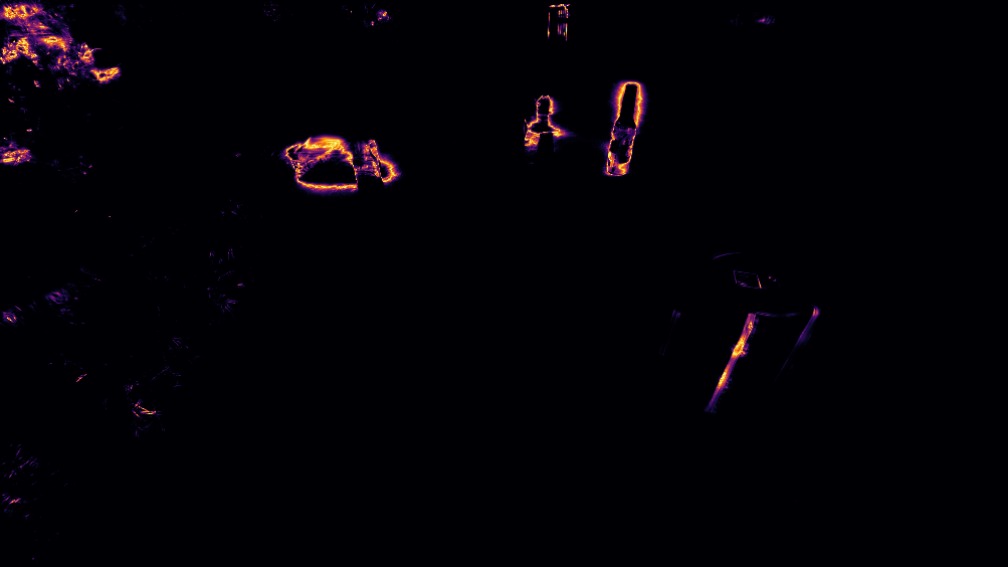}{0.624}{0.769}{1.2pt}
\marginrow{Pots}{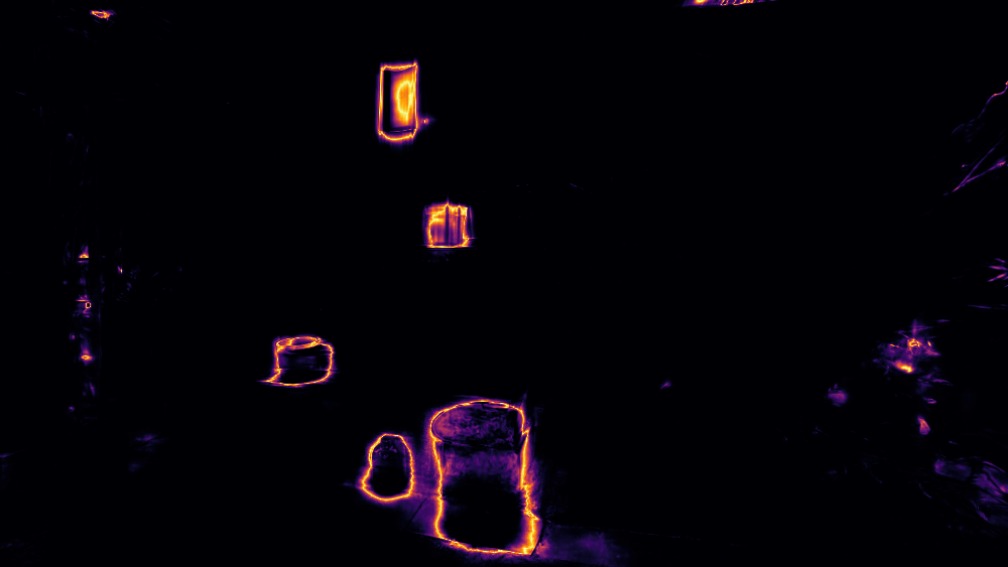}{0.778}{0.875}{1.2pt}
\marginrow{Cantina}{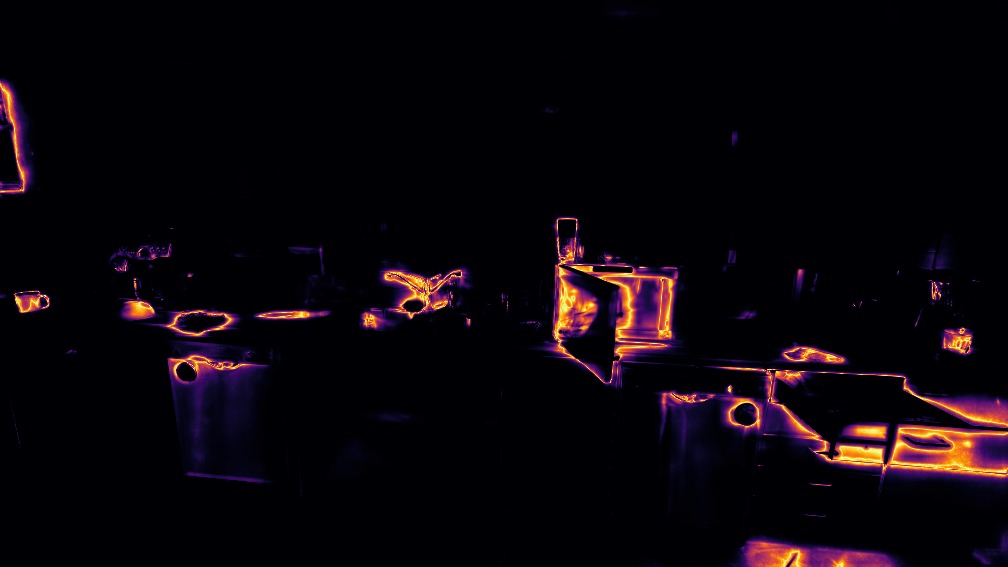}{0.610}{0.758}{1.2pt}
\marginrow{Meeting room}{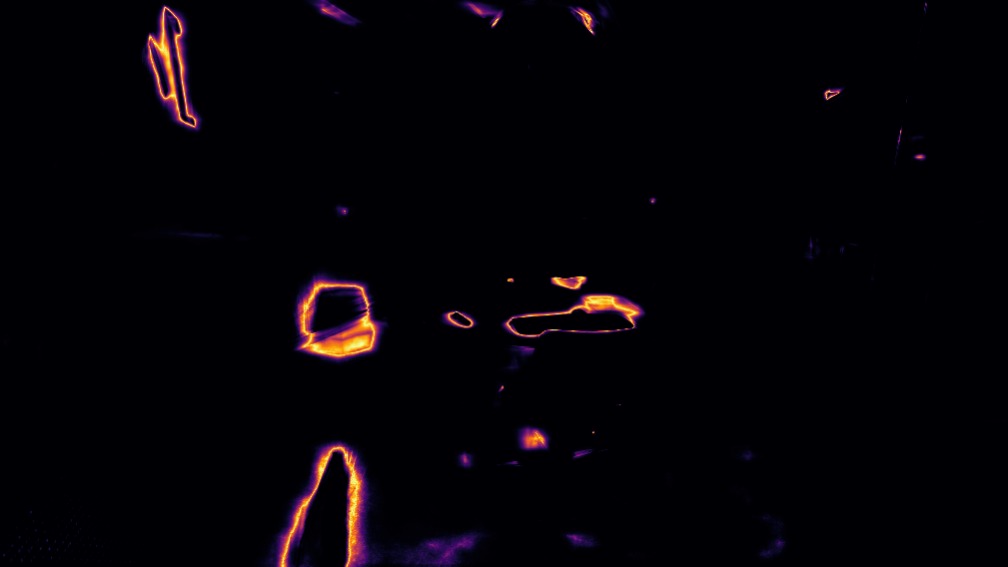}{0.687}{0.815}{0pt}
\end{tabular}

\vspace{4pt}
{\footnotesize
\marginkey{white}{True positive}\qquad
\marginkey{black}{True negative}\qquad
\marginkey{qualfn}{False negative}\qquad
\marginkey{qualfp}{False positive}}

\caption{Decision margins of O\protect\nobreakdash-SCD++ on PASLCD~\cite{galappaththige2025multi}. The margin map is bright where the margin is small, that is, where a slightly different threshold would change the label. These regions coincide with the misclassified pixels of the prediction, change boundaries, and missed regions.}
\label{fig:margin}
\end{figure*}

\begin{table}[t]
\centering
\footnotesize
\setlength{\tabcolsep}{2pt}
\caption{Sensitivity to the order of the inference frames on PASLCD~\cite{galappaththige2025multi}, over the incoming order and three random permutations. mIoU is the mean $\pm$ standard deviation over the four orderings; mask agreement is the mean per-frame IoU between the masks obtained under each permutation and under the incoming order.}
\label{tab:order}
\begin{tabular*}{0.9\columnwidth}{@{\extracolsep{\fill}}llcc}
\toprule
Method & Setting & mIoU & Mask agreement \\
\midrule
O\nobreakdash-SCD~\cite{galappaththige2025changes} & online & $0.4968 \pm 0.0057$ & 0.674 \\
O\nobreakdash-SCD~\cite{galappaththige2025changes} & retrospective & $0.5089 \pm 0.0019$ & 0.853 \\
\midrule
O\nobreakdash-SCD++ & online & $0.5767 \pm 0.0007$ & 0.883 \\
O\nobreakdash-SCD++ & retrospective & $\mathbf{0.5884 \pm 0.0000}$ & \textbf{0.999} \\
\bottomrule
\end{tabular*}
\end{table}

\noindent\textbf{Frame-order invariance.} \Cref{prop:exact} implies that the final state of O\nobreakdash-SCD++ depends on which frames were observed but not on the order in which they arrived. This guarantee holds for the state once every frame has been observed, and hence for retrospective inference. It does not extend to online masks, since the mask of each frame is inferred from only the frames that preceded it, and which frames those are depends on the order. We empirically validate this by processing each sequence in its incoming order and in three random permutations. Pose estimation is seeded per frame so that every frame receives the same pose in every ordering, which is why the values in \cref{tab:order} differ slightly from those in \cref{tab:paslcd,tab:retrospective}. \Cref{tab:order} reports the results. The retrospective masks of O\nobreakdash-SCD++ are identical across orderings up to floating-point noise, with the same mIoU to four decimal places and a mask agreement of $0.999$. Those of O\nobreakdash-SCD are not: although their mIoU varies only moderately, the masks themselves differ substantially, with a mask agreement of $0.853$. As expected, online masks depend on the order for both methods, but those of O\nobreakdash-SCD++ vary about eight times less than those of O\nobreakdash-SCD.

\noindent\textbf{Predictive photometric uncertainty guidance.} Rendering artifacts in the reference image can cause spurious cues~\cite{galappaththige2026predictive}. U-3DGS~\cite{galappaththige2026predictive} fits a per-primitive uncertainty channel to a trained 3DGS map from its reconstruction residuals; rendered at the estimated pose of an inference image and clipped to $[0,1]$, it gives a per-pixel uncertainty $u_t(p)$, which attenuates the cue of \cref{eq:cue}:
\begin{equation}
\tilde{c}_t(p) = c_t(p)\,\bigl(1 - u_t(p)\bigr).
\label{eq:uncertainty}
\end{equation}
Both the uncertainty estimate and \cref{eq:uncertainty} are contributions of~\cite{galappaththige2026predictive}, which we apply unchanged: the channel is fitted once to each reference map before the revisit, and $\tilde{c}_t$ replaces $c_t$ in both methods. \Cref{tab:uncertainty} shows that the uncertainty improves both methods at no measurable cost in inference speed; on PASLCD, O\nobreakdash-SCD++ gains considerably more. \Cref{fig:uncertainty} shows qualitative examples.

\noindent\textbf{Runtime.} \Cref{tab:runtime} breaks the per-frame processing time down by module. In O\nobreakdash-SCD, change detection takes $64\%$ of the per-frame processing time, most of it in multi-view fusion. The replay-free fusion of O\nobreakdash-SCD++ is about $20\times$ faster, so change detection falls to $23\%$. Change detection is therefore no longer the bottleneck; further speed-ups could come from the pose estimation module.

\subsection{Selective Prediction from Decision Margins}
\label{sec:selective}

We test whether the decision margin of \cref{sec:margin} identifies the pixels that a mask labels incorrectly, and compare it with photometric uncertainty and with the fixed-cutoff distance of O\nobreakdash-SCD. \Cref{fig:sparsification} follows the sparsification protocol of~\cite{ilg2018uncertainty}, with IoU in place of their flow error: it withholds a growing fraction of pixels in each frame and reports the mIoU of the online masks on the pixels that remain. Withholding the pixels of smallest decision margin raises mIoU from $0.575$ to $0.831$ at $10\%$ and $0.879$ at $20\%$ withheld, whereas random withholding leaves it unchanged (area under the sparsification error, AUSE, $0.116$ against $0.409$). Withholding by photometric uncertainty~\cite{galappaththige2026predictive} is far less effective ($0.598$ at $10\%$): it indicates where the \emph{reference map} is unreliable rather than where the \emph{decision} is. An oracle withholding misclassified pixels first bounds the gain.

\noindent\textbf{What the margin measures.} The margin is a confidence in the decision. It measures how far the fused evidence at a pixel lies from the point at which the label-free objective would label it differently. A small margin marks a pixel whose label a slightly different weighing of the evidence would flip. \Cref{fig:margin} shows that small margins concentrate on the misclassified pixels, along object boundaries and within missed regions. These small boundary margins are largely a consequence of how the cues are generated: the semantic cue comes from a feature backbone whose resolution is bounded by its patch size~\cite{galappaththige2026pixels}, and the photometric cue is computed over local windows. The rendered map further spreads each primitive's evidence over its footprint. Pixels near object boundaries therefore receive intermediate scores and hence small margins.

\noindent\textbf{Margin and photometric uncertainty.} The two signals answer different questions. Photometric uncertainty~\cite{galappaththige2026predictive} is a property of the reference map: it predicts where a rendering is likely to deviate from the true scene, independently of any change, and is therefore most useful upstream to attenuate cues produced by rendering artifacts (\cref{tab:uncertainty}). The margin is a property of the decision: it is computed after fusion and thresholding and reflects how decisively the accumulated evidence separates change from no change. \Cref{fig:sparsification} shows the difference, as ranking by uncertainty barely improves the retained region whereas ranking by margin does. The two are complementary: uncertainty cleans the evidence, and the margin reports how far each resulting decision can be trusted.

\noindent\textbf{Why earlier methods lack a margin.} A margin is informative only if the scores near the threshold are the uncertain ones. Label-free SCD methods such as MV3DCD~\cite{galappaththige2025multi} and O\nobreakdash-SCD~\cite{galappaththige2025changes} binarize at a fixed cutoff of $0.5$. As \cref{sec:oscd} explains, detected change in O\nobreakdash-SCD settles only modestly above the cutoff. Correct detections and errors therefore lie at similar distances from the cutoff, and ranking by this distance withholds correct detections more readily than errors: withholding $10\%$ of pixels lowers the mIoU of O\nobreakdash-SCD from $0.489$ to $0.303$ (\cref{fig:sparsification}). In O\nobreakdash-SCD++, the fused score is a weighted average of cues rather than the output of an optimizer, and the threshold is placed by the decision objective, so the distance to it separates decisive from marginal decisions at no cost.\Cref{sec:limitations} discusses how this confidence can be used.

\section{Limitations and Future Work}
\label{sec:limitations}

\noindent\textbf{Change cues.} O\nobreakdash-SCD++ fuses the evidence its cues provide and cannot recover changes that the cues miss. This limitation stems from the vision foundation models used for cue extraction~\cite{ravi2024sam} and is shared with MV3DCD~\cite{galappaththige2025multi} and O\nobreakdash-SCD~\cite{galappaththige2025changes}: semantic cues are largely insensitive to surface-level changes on semantically similar objects, such as a change of colour. Complementing them with photometric cues~\cite{galappaththige2025multi} and fusing both in continuous form~\cite{galappaththige2025changes} mitigates this, as for the T-shaped object that turns from red to blue in the Meeting room of \cref{fig:uncertainty}, but some appearance changes are still missed, such as the bowl in Cantina that turns from pink to blue in \cref{fig:margin}. The resolution of the cues also limits the precision of the masks at object boundaries (\cref{sec:selective}). Cue extraction is modular, so O\nobreakdash-SCD++ benefits directly from advances in vision foundation models. Designing better cues, and combining cues from several modalities, such as semantic, photometric and geometric evidence, is a promising direction, since the fusion accepts any per-pixel cue in $[0,1]$. Comparing the scenes in the primitives rather than in the image~\cite{galappaththige2026pixels} offers an alternative to image-space cues.

\noindent\textbf{Reference map.} As in other 3DGS-based change detection methods~\cite{galappaththige2025multi,lu20253dgs,zhou2026d,jiang2025gaussian,galappaththige2026pixels,lai2026plenoci,galappaththige2025changes}, change can be inferred only relative to the reference surface, so changes in regions that the reference map does not cover cannot be detected. O\nobreakdash-SCD++ therefore assumes that the reference traversal observes the scene from enough viewpoints, with enough overlap, to reconstruct an accurate 3DGS map (\cref{sec:problem}). This matches autonomous inspection, where an asset is surveyed repeatedly and a reference map or digital twin already exists, but it makes O\nobreakdash-SCD++ hard to apply to benchmarks such as SceneDiff~\cite{wu2025scenediff}, whose short video clips with narrow camera baselines do not support a reliable 3DGS reconstruction. As in other SCD methods, the scene is also assumed to be static during each capture. Extending O\nobreakdash-SCD++ to captures with moving objects, and keeping the reference map or digital twin current between revisits, are natural directions. Towards the latter, the masks of O\nobreakdash-SCD++ can drive the representation update of O\nobreakdash-SCD~\cite{galappaththige2025changes} unchanged (Appendix~\ref{app:update}).

\noindent\textbf{Front end and scale.} The pose front end is adopted from O\nobreakdash-SCD. XFeat~\cite{potje2024xfeat} is a lightweight descriptor designed for efficiency on real images, which suits online operation in real environments, but registration with it can struggle under large in-plane rotations, requires inputs of at least VGA resolution~\cite{galappaththige2025changes}, and in our experience does not yield robust matches on synthetic images such as those of ChangeSim~\cite{park2021changesim}. The pose front end is modular, as Appendix~\ref{app:changesim} shows with ALIKED~\cite{zhao2023aliked} and LightGlue~\cite{lightglue_2023} on ChangeSim. Advances in lightweight feature matching would therefore carry over directly and improve throughput, since pose estimation now dominates the runtime (\cref{tab:runtime}). Appendix~\ref{app:keyframes} shows that restricting retrieval to keyframes selected adaptively from the reference frames already raises throughput by $9\%$ without loss of accuracy. The cost of rendering and evidence transfer grows with the number of primitives, which scalable representations~\cite{Kerbl2024AH3,fastgs,HansonSpeedy} could bound in large scenes.

\noindent\textbf{Decision margin.} The margin ranks decisions by confidence but is not a calibrated probability. Even as a ranking, it can inform downstream policies. A system can report only the changes it is confident about, pass the least confident regions to an operator for review, or, while still on site, plan a path that re-observes them from new viewpoints, which turns the margin into a criterion for active view selection or active mapping in changing environments. Calibrating the margin and feeding it back into the fusion, for example to weight observations, are natural extensions.

\section{Conclusion}
\label{sec:conclusion}

This work shows that the accuracy--speed trade-off of iterative fusion in online scene change detection can be removed. O\nobreakdash-SCD++ does so by replacing the iterative, replay-based fusion of O\nobreakdash-SCD with a direct estimator built on the linearity of splatting, in which no stage requires iterative optimization. The recursion reproduces the aggregate over all observed frames from two scalars per primitive, and the decision is the exact minimizer of a label-free objective over threshold masks. An extended evaluation analyses each component, its behaviour across scene types, lighting and frame orderings, its gains from predictive photometric uncertainty guidance, and the decision margins. O\nobreakdash-SCD++ sets a new state of the art for online SCD on PASLCD, CL-Splats and ChangeSim. On PASLCD it raises mIoU from $0.489$ to $0.575$ over O\nobreakdash-SCD at $2.2\times$ its throughput, and surpasses the offline result of O\nobreakdash-SCD while remaining causal. The state reached after a sequence brings further properties: it does not depend on the order in which frames arrive and, re-applied to every frame, improves on the online masks without refinement. The exact decision also provides a per-pixel margin that ranks decisions by confidence. We believe this brings online SCD a step closer to real-world, real-time deployment, where an agent must report not only what has changed but how certain each decision is, while it is still on site.

\section*{Acknowledgments}

This work was supported by the Australian Research Council Research Hub in Intelligent Robotic Systems for Real-Time Asset Management (ARIAM) (IH210100030) and Abyss Solutions. C.J.G., N.S. and D.M. also acknowledge ongoing support from the QUT Centre for Robotics.

\bibliographystyle{IEEEtran}
\bibliography{references}

\begin{thebibliography}{10}
\providecommand{\url}[1]{#1}
\csname url@samestyle\endcsname
\providecommand{\newblock}{\relax}
\providecommand{\bibinfo}[2]{#2}
\providecommand{\BIBentrySTDinterwordspacing}{\spaceskip=0pt\relax}
\providecommand{\BIBentryALTinterwordstretchfactor}{4}
\providecommand{\BIBentryALTinterwordspacing}{\spaceskip=\fontdimen2\font plus
\BIBentryALTinterwordstretchfactor\fontdimen3\font minus
  \fontdimen4\font\relax}
\providecommand{\BIBforeignlanguage}[2]{{%
\expandafter\ifx\csname l@#1\endcsname\relax
\typeout{** WARNING: IEEEtran.bst: No hyphenation pattern has been}%
\typeout{** loaded for the language `#1'. Using the pattern for}%
\typeout{** the default language instead.}%
\else
\language=\csname l@#1\endcsname
\fi
#2}}
\providecommand{\BIBdecl}{\relax}
\BIBdecl

\bibitem{sakurada_change_2015}
K.~Sakurada and T.~Okatani, ``Change detection from a street image pair using
  {CNN} features and superpixel segmentation,'' in \emph{Proceedings of the
  British Machine Vision Conference}, 2015, pp. 61.1--61.12.

\bibitem{alcantarilla2018street}
P.~F. Alcantarilla, S.~Stent, G.~Ros, R.~Arroyo, and R.~Gherardi, ``Street-view
  change detection with deconvolutional networks,'' \emph{Autonomous Robots},
  vol.~42, no.~7, pp. 1301--1322, 2018.

\bibitem{daudt2018fully}
R.~C. Daudt, B.~Le~Saux, and A.~Boulch, ``Fully convolutional siamese networks
  for change detection,'' in \emph{Proceedings of the IEEE International
  Conference on Image Processing}, 2018, pp. 4063--4067.

\bibitem{cyws2d}
R.~Sachdeva and A.~Zisserman, ``The change you want to see,'' in
  \emph{Proceedings of the IEEE/CVF Winter Conference on Applications of
  Computer Vision}, 2023, pp. 3993--4002.

\bibitem{galappaththige2025multi}
C.~J. Galappaththige, J.~Lai, L.~Windrim, D.~G. Dansereau, N.~Suenderhauf, and
  D.~Miller, ``Multi-view pose-agnostic change localization with zero labels,''
  in \emph{Proceedings of the IEEE/CVF Conference on Computer Vision and
  Pattern Recognition}, 2025.

\bibitem{lu20253dgs}
Z.~Lu, J.~Ye, and J.~Leonard, ``{3DGS-CD}: {3D} {Gaussian} splatting-based
  change detection for physical object rearrangement,'' \emph{IEEE Robotics and
  Automation Letters}, 2025.

\bibitem{galappaththige2026pixels}
C.~J. Galappaththige, J.~Lai, T.~Patten, D.~Dansereau, N.~Suenderhauf, and
  D.~Miller, ``From pixels to primitives: Scene change detection in {3D}
  {Gaussian} splatting,'' \emph{arXiv preprint arXiv:2605.07203}, 2026.

\bibitem{jiang2025gaussian}
B.~Jiang, R.~Huang, Q.~Zhao, and Y.~Zhang, ``Gaussian difference: Find any
  change instance in {3D} scenes,'' in \emph{Proceedings of the IEEE
  International Conference on Acoustics, Speech and Signal Processing}, 2025,
  pp. 1--5.

\bibitem{zhou2026d}
Z.~Zhou, J.~Ni, S.~Zhang, Y.~Chen, and S.~Huang, ``{3D} scene change modeling
  with consistent multi-view aggregation,'' in \emph{Proceedings of the
  International Conference on 3D Vision}, 2026.

\bibitem{galappaththige2025changes}
C.~J. Galappaththige, J.~Lai, L.~Windrim, D.~G. Dansereau, N.~Suenderhauf, and
  D.~Miller, ``Changes in real time: Online scene change detection with
  multi-view fusion,'' in \emph{Proceedings of the IEEE/CVF Conference on
  Computer Vision and Pattern Recognition}, 2026.

\bibitem{park2021changesim}
J.-M. Park, J.-H. Jang, S.-M. Yoo, S.-K. Lee, U.-H. Kim, and J.-H. Kim,
  ``{ChangeSim}: Towards end-to-end online scene change detection in industrial
  indoor environments,'' in \emph{Proceedings of the IEEE/RSJ International
  Conference on Intelligent Robots and Systems}, 2021, pp. 8578--8585.

\bibitem{kerbl20233d}
B.~Kerbl, G.~Kopanas, T.~Leimk{\"u}hler, and G.~Drettakis, ``{3D} {Gaussian}
  splatting for real-time radiance field rendering.'' \emph{ACM Transactions on
  Graphics}, vol.~42, no.~4, pp. 139--1, 2023.

\bibitem{galappaththige2026predictive}
C.~J. Galappaththige, T.~Gottwald, P.~Stehr, E.~Heinert, N.~Suenderhauf,
  D.~Miller, and M.~Rottmann, ``Predictive photometric uncertainty in
  {Gaussian} splatting for novel view synthesis,'' in \emph{Proceedings of the
  European Conference on Computer Vision}.\hskip 1em plus 0.5em minus
  0.4em\relax Springer, 2026, pp. 482--500.

\bibitem{marrie2025ludvig}
J.~Marrie, R.~Menegaux, M.~Arbel, D.~Larlus, and J.~Mairal, ``{LUDVIG}:
  Learning-free uplifting of {2D} visual features to {Gaussian} splatting
  scenes,'' in \emph{Proceedings of the IEEE/CVF International Conference on
  Computer Vision (ICCV)}, 2025, pp. 7440--7450.

\bibitem{joseph2025gradient}
J.~Joseph, B.~Amrutur, and S.~Bhatnagar, ``Gradient-weighted feature
  back-projection: A fast alternative to feature distillation in {3D}
  {Gaussian} splatting,'' in \emph{Proceedings of SIGGRAPH Asia}, 2025.

\bibitem{varghese2018changenet}
A.~Varghese, J.~Gubbi, A.~Ramaswamy, and P.~Balamuralidhar, ``{ChangeNet}: A
  deep learning architecture for visual change detection,'' in
  \emph{Proceedings of the European Conference on Computer Vision Workshops},
  2018.

\bibitem{sakurada2020weakly}
K.~Sakurada, M.~Shibuya, and W.~Wang, ``Weakly supervised silhouette-based
  semantic scene change detection,'' in \emph{Proceedings of the IEEE
  International Conference on Robotics and Automation}, 2020, pp. 6861--6867.

\bibitem{lin2025robust}
C.-J. Lin, S.~Garg, T.-J. Chin, and F.~Dayoub, ``Robust scene change detection
  using visual foundation models and cross-attention mechanisms,'' in
  \emph{Proceedings of the IEEE International Conference on Robotics and
  Automation}, 2025, pp. 8337--8343.

\bibitem{cyws3d}
R.~Sachdeva and A.~Zisserman, ``The change you want to see (now in {3D}),'' in
  \emph{Proceedings of the IEEE/CVF International Conference on Computer Vision
  Workshops}, 2023, pp. 2060--2069.

\bibitem{lee2024semi}
S.~Lee and J.-H. Kim, ``Semi-supervised scene change detection by distillation
  from feature-metric alignment,'' in \emph{Proceedings of the IEEE/CVF Winter
  Conference on Applications of Computer Vision}, 2024, pp. 1215--1224.

\bibitem{gulrajani2021search}
I.~Gulrajani and D.~Lopez-Paz, ``In search of lost domain generalization,'' in
  \emph{International Conference on Learning Representations}, 2021.

\bibitem{galappaththige2025domain}
C.~J. Galappaththige, Z.~Izzo, X.~He, H.~Zhou, and M.~H. Khan, ``Domain-guided
  weight modulation for semi-supervised domain generalization,'' in
  \emph{Proceedings of the IEEE/CVF Winter Conference on Applications of
  Computer Vision}, 2025, pp. 6495--6505.

\bibitem{kim2025towards}
J.-W. Kim and U.-H. Kim, ``Towards generalizable scene change detection,'' in
  \emph{Proceedings of the IEEE/CVF Conference on Computer Vision and Pattern
  Recognition}, 2025, pp. 24\,463--24\,473.

\bibitem{kannanZero}
S.~S. Kannan and B.-C. Min, ``{ZeroSCD}: Zero-shot street scene change
  detection,'' in \emph{Proceedings of the IEEE International Conference on
  Robotics and Automation}, 2025, pp. 4665--4671.

\bibitem{cho2025zero}
K.~Cho, D.~Y. Kim, and E.~Kim, ``Zero-shot scene change detection,'' in
  \emph{Proceedings of the AAAI Conference on Artificial Intelligence}, 2025,
  pp. 2509--2517.

\bibitem{alpherts2025emplace}
T.~Alpherts, S.~Ghebreab, and N.~van Noord, ``{EMPLACE}: Self-supervised urban
  scene change detection,'' in \emph{Proceedings of the AAAI Conference on
  Artificial Intelligence}, 2025, pp. 1737--1745.

\bibitem{kirillov2023segment}
A.~Kirillov, E.~Mintun, N.~Ravi, H.~Mao, C.~Rolland, L.~Gustafson, T.~Xiao,
  S.~Whitehead, A.~C. Berg, W.-Y. Lo \emph{et~al.}, ``Segment anything,'' in
  \emph{Proceedings of the IEEE/CVF International Conference on Computer
  Vision}, 2023, pp. 4015--4026.

\bibitem{oquab2023dinov2}
M.~Oquab, T.~Darcet, T.~Moutakanni, H.~Vo, M.~Szafraniec, V.~Khalidov,
  P.~Fernandez, D.~Haziza, F.~Massa, A.~El-Nouby \emph{et~al.}, ``{DINOv2}:
  Learning robust visual features without supervision,'' \emph{Transactions on
  Machine Learning Research}, 2024.

\bibitem{ravi2024sam}
N.~Ravi, V.~Gabeur, Y.-T. Hu, R.~Hu, C.~Ryali, T.~Ma, H.~Khedr, R.~R{\"a}dle,
  C.~Rolland, L.~Gustafson \emph{et~al.}, ``{SAM} 2: Segment anything in images
  and videos,'' in \emph{Proceedings of the International Conference on
  Learning Representations}, 2025.

\bibitem{mildenhall2021nerf}
B.~Mildenhall, P.~P. Srinivasan, M.~Tancik, J.~T. Barron, R.~Ramamoorthi, and
  R.~Ng, ``{NeRF}: Representing scenes as neural radiance fields for view
  synthesis,'' \emph{Communications of the ACM}, vol.~65, no.~1, pp. 99--106,
  2021.

\bibitem{zhou2023pad}
Q.~Zhou, W.~Li, L.~Jiang, G.~Wang, G.~Zhou, S.~Zhang, and H.~Zhao, ``{PAD}: A
  dataset and benchmark for pose-agnostic anomaly detection,'' in
  \emph{Advances in Neural Information Processing Systems}, 2023, pp.
  44\,558--44\,571.

\bibitem{kruse2024splatpose}
M.~Kruse, M.~Rudolph, D.~Woiwode, and B.~Rosenhahn, ``{SplatPose \& Detect}:
  Pose-agnostic {3D} anomaly detection,'' in \emph{Proceedings of the IEEE/CVF
  Conference on Computer Vision and Pattern Recognition Workshops}, 2024, pp.
  3950--3960.

\bibitem{wu20244d}
G.~Wu, T.~Yi, J.~Fang, L.~Xie, X.~Zhang, W.~Wei, W.~Liu, Q.~Tian, and X.~Wang,
  ``{4D} {Gaussian} splatting for real-time dynamic scene rendering,'' in
  \emph{Proceedings of the IEEE/CVF Conference on Computer Vision and Pattern
  Recognition}, 2024, pp. 20\,310--20\,320.

\bibitem{huang2025c}
R.~Huang, H.~Tao, B.~Jiang, Q.~Zhao, L.~Wan, and Q.~Guo, ``{C-NeRF}:
  Representing scene changes as directional consistency difference-based
  {NeRF},'' \emph{IEEE Transactions on Image Processing}, vol.~34, pp.
  7502--7513, 2025.

\bibitem{lai2026plenoci}
J.~Lai, C.~J. Galappaththige, N.~S{\"u}nderhauf, D.~Miller, and D.~G.
  Dansereau, ``{PlenoCI}: Plenoptic characteristics for view dependence aware
  change classification,'' \emph{arXiv preprint arXiv:2609.28930}, 2026.

\bibitem{friedlander2026goldilocs}
A.~Friedlander, A.~Shamir, and O.~Fried, ``{GOLDILOCS}: General object-level
  detection and labeling of changes in scenes,'' in \emph{Proceedings of the
  International Conference on Learning Representations}, 2026.

\bibitem{wu2025scenediff}
Y.~Wu, C.-h. Lin, H.~Che, A.~Tiwari, C.~Zou, S.~Wang, and D.~Hoiem,
  ``{SceneDiff}: A benchmark and method for multiview object change
  detection,'' in \emph{Proceedings of the European Conference on Computer
  Vision}, 2026.

\bibitem{wang2026pi}
Y.~Wang, J.~Zhou, H.~Zhu, W.~Chang, Y.~Zhou, Z.~Li, J.~Chen, J.~Pang, C.~Shen,
  and T.~He, ``Permutation-equivariant visual geometry learning,'' in
  \emph{Proceedings of the International Conference on Learning
  Representations}, 2026, pp. 10\,481--10\,497.

\bibitem{wang_dust3r_2024}
S.~Wang, V.~Leroy, Y.~Cabon, B.~Chidlovskii, and J.~Revaud, ``{DUSt3R}:
  Geometric {3D} vision made easy,'' in \emph{Proceedings of the IEEE/CVF
  Conference on Computer Vision and Pattern Recognition}, 2024, pp.
  20\,697--20\,709.

\bibitem{yoon2026vscd}
J.~Yoon and U.-H. Kim, ``{VSCD}: Video-based scene change detection in
  unaligned scenes,'' in \emph{Proceedings of the International Conference on
  Machine Learning}, 2026.

\bibitem{schonberger2016structure}
J.~L. Schonberger and J.-M. Frahm, ``Structure-from-motion revisited,'' in
  \emph{Proceedings of the IEEE Conference on Computer Vision and Pattern
  Recognition}, 2016, pp. 4104--4113.

\bibitem{labbe2019rtab}
M.~Labb{\'e} and F.~Michaud, ``{RTAB-Map} as an open-source lidar and visual
  simultaneous localization and mapping library for large-scale and long-term
  online operation,'' \emph{Journal of Field Robotics}, vol.~36, no.~2, pp.
  416--446, 2019.

\bibitem{yeh2026towards}
Y.-H. Yeh, C.-Y. Huang, K.-W. Chen, and L.-Y. Lu, ``Towards practical scene
  change detection: A fast, unaligned video framework via spatiotemporal
  alignment,'' in \emph{Proceedings of the IEEE/RSJ International Conference on
  Intelligent Robots and Systems}, 2026.

\bibitem{izquierdo2024optimal}
S.~Izquierdo and J.~Civera, ``Optimal transport aggregation for visual place
  recognition,'' in \emph{Proceedings of the IEEE/CVF Conference on Computer
  Vision and Pattern Recognition}, 2024, pp. 17\,658--17\,668.

\bibitem{liu2024splatpose+}
Y.~Liu, Y.~S. Hu, Y.~Chen, and J.~Zelek, ``{SplatPose+}: Real-time image-based
  pose-agnostic {3D} anomaly detection,'' in \emph{Proceedings of the European
  Conference on Computer Vision Workshops}, 2024, pp. 378--391.

\bibitem{schmid2022panoptic}
L.~Schmid, J.~Delmerico, J.~L. Sch{\"o}nberger, J.~Nieto, M.~Pollefeys,
  R.~Siegwart, and C.~Cadena, ``Panoptic multi-{TSDFs}: A flexible
  representation for online multi-resolution volumetric mapping and long-term
  dynamic scene consistency,'' in \emph{Proceedings of the IEEE International
  Conference on Robotics and Automation}, 2022, pp. 8018--8024.

\bibitem{fu2022planesdf}
J.~Fu, C.~Lin, Y.~Taguchi, A.~Cohen, Y.~Zhang, S.~Mylabathula, and J.~J.
  Leonard, ``{PlaneSDF}-based change detection for long-term dense mapping,''
  \emph{IEEE Robotics and Automation Letters}, vol.~7, no.~4, pp. 9667--9674,
  2022.

\bibitem{qian2023pov}
J.~Qian, V.~Chatrath, J.~Servos, A.~Mavrinac, W.~Burgard, S.~L. Waslander, and
  A.~P. Schoellig, ``{POV-SLAM}: Probabilistic object-aware variational {SLAM}
  in semi-static environments,'' in \emph{Proceedings of Robotics: Science and
  Systems}, 2023.

\bibitem{zhu2024living}
L.~Zhu, S.~Huang, K.~Schindler, and I.~Armeni, ``Living scenes: Multi-object
  relocalization and reconstruction in changing {3D} environments,'' in
  \emph{Proceedings of the IEEE/CVF Conference on Computer Vision and Pattern
  Recognition}, 2024, pp. 28\,014--28\,024.

\bibitem{yugay2025gaussian}
V.~Yugay, T.~Kersten, L.~Carlone, T.~Gevers, M.~R. Oswald, and L.~Schmid,
  ``{Gaussian} mapping for evolving scenes,'' in \emph{Proceedings of the
  IEEE/CVF Conference on Computer Vision and Pattern Recognition}, 2026, pp.
  18\,903--18\,912.

\bibitem{zeng2025gaussianupdate}
L.~Zeng, B.~Zhao, J.~Hu, X.~Shen, Z.~Dang, H.~Bao, and Z.~Cui,
  ``{GaussianUpdate}: Continual {3D} {Gaussian} splatting update for changing
  environments,'' in \emph{Proceedings of the IEEE/CVF International Conference
  on Computer Vision}, 2025.

\bibitem{ackermann2025clsplats}
J.~Ackermann, J.~Kulhanek, S.~Cai, X.~Haofei, M.~Pollefeys, G.~Wetzstein,
  L.~Guibas, and S.~Peng, ``{CL-Splats}: Continual learning of {Gaussian}
  splatting with local optimization,'' in \emph{Proceedings of the IEEE/CVF
  International Conference on Computer Vision}, 2025.

\bibitem{schmid2024khronos}
L.~Schmid, M.~Abate, Y.~Chang, and L.~Carlone, ``Khronos: A unified approach
  for spatio-temporal metric-semantic {SLAM} in dynamic environments,'' in
  \emph{Proceedings of Robotics: Science and Systems}, 2024.

\bibitem{jiang2024fisherrf}
W.~Jiang, B.~Lei, and K.~Daniilidis, ``{FisherRF}: Active view selection and
  mapping with radiance fields using {Fisher} information,'' in
  \emph{Proceedings of the European Conference on Computer Vision}, 2024, pp.
  422--440.

\bibitem{potje2024xfeat}
G.~Potje, F.~Cadar, A.~Araujo, R.~Martins, and E.~R. Nascimento, ``{XFeat}:
  Accelerated features for lightweight image matching,'' in \emph{Proceedings
  of the IEEE/CVF Conference on Computer Vision and Pattern Recognition}, 2024,
  pp. 2682--2691.

\bibitem{lepetit2009epnp}
V.~Lepetit, F.~Moreno-Noguer, and P.~Fua, ``{EPnP}: An accurate {O(n)} solution
  to the {PnP} problem,'' \emph{International Journal of Computer Vision},
  vol.~81, no.~2, pp. 155--166, 2009.

\bibitem{fischler1981ransac}
M.~A. Fischler and R.~C. Bolles, ``Random sample consensus: A paradigm for
  model fitting with applications to image analysis and automated
  cartography,'' \emph{Communications of the ACM}, vol.~24, no.~6, pp.
  381--395, 1981.

\bibitem{levenberg1944method}
K.~Levenberg, ``A method for the solution of certain non-linear problems in
  least squares,'' \emph{Quarterly of Applied Mathematics}, vol.~2, no.~2, pp.
  164--168, 1944.

\bibitem{meuleman2025fly}
A.~Meuleman, I.~Shah, A.~Lanvin, B.~Kerbl, and G.~Drettakis, ``On-the-fly
  reconstruction for large-scale novel view synthesis from unposed images,''
  \emph{ACM Transactions on Graphics}, vol.~44, no.~4, pp. 1--14, 2025.

\bibitem{wang_image_2004}
Z.~Wang, A.~Bovik, H.~Sheikh, and E.~Simoncelli, ``Image quality assessment:
  From error visibility to structural similarity,'' \emph{IEEE Transactions on
  Image Processing}, vol.~13, no.~4, pp. 600--612, 2004.

\bibitem{zwicker2002ewa}
M.~Zwicker, H.~Pfister, J.~Van~Baar, and M.~Gross, ``{EWA} splatting,''
  \emph{IEEE Transactions on Visualization and Computer Graphics}, vol.~8,
  no.~03, pp. 223--238, 2002.

\bibitem{curless1996volumetric}
B.~Curless and M.~Levoy, ``A volumetric method for building complex models from
  range images,'' in \emph{Proceedings of the Annual Conference on Computer
  Graphics and Interactive Techniques}, 1996, pp. 303--312.

\bibitem{elyaniv2010foundations}
R.~El-Yaniv and Y.~Wiener, ``On the foundations of noise-free selective
  classification,'' \emph{Journal of Machine Learning Research}, vol.~11, pp.
  1605--1641, 2010.

\bibitem{geifman2017selective}
Y.~Geifman and R.~El-Yaniv, ``Selective classification for deep neural
  networks,'' in \emph{Advances in Neural Information Processing Systems},
  2017.

\bibitem{Barron2021MipNeRF3U}
J.~T. Barron, B.~Mildenhall, D.~Verbin, P.~P. Srinivasan, and P.~Hedman,
  ``{Mip-NeRF} 360: Unbounded anti-aliased neural radiance fields,'' in
  \emph{Proceedings of the IEEE/CVF Conference on Computer Vision and Pattern
  Recognition}, 2022, pp. 5460--5469.

\bibitem{lin2014microsoft}
T.-Y. Lin, M.~Maire, S.~Belongie, J.~Hays, P.~Perona, D.~Ramanan,
  P.~Doll{\'a}r, and C.~L. Zitnick, ``Microsoft {COCO}: Common objects in
  context,'' in \emph{Proceedings of the European Conference on Computer
  Vision}, 2014, pp. 740--755.

\bibitem{otsu1979threshold}
N.~Otsu, ``A threshold selection method from gray-level histograms,''
  \emph{IEEE Transactions on Systems, Man, and Cybernetics}, vol.~9, no.~1, pp.
  62--66, 1979.

\bibitem{ilg2018uncertainty}
E.~Ilg, {\"O}.~{\c{C}}i{\c{c}}ek, S.~Galesso, A.~Klein, O.~Makansi, F.~Hutter,
  and T.~Brox, ``Uncertainty estimates and multi-hypotheses networks for
  optical flow,'' in \emph{Proceedings of the European Conference on Computer
  Vision (ECCV)}, 2018.

\bibitem{zhao2023aliked}
X.~Zhao, X.~Wu, W.~Chen, P.~C.~Y. Chen, Q.~Xu, and Z.~Li, ``{ALIKED}: A lighter
  keypoint and descriptor extraction network via deformable transformation,''
  \emph{IEEE Transactions on Instrumentation and Measurement}, vol.~72, pp.
  1--16, 2023.

\bibitem{lightglue_2023}
P.~Lindenberger, P.-E. Sarlin, and M.~Pollefeys, ``{LightGlue}: Local feature
  matching at light speed,'' in \emph{Proceedings of the IEEE/CVF International
  Conference on Computer Vision}, 2023, pp. 17\,581--17\,592.

\bibitem{Kerbl2024AH3}
B.~Kerbl, A.~Meuleman, G.~Kopanas, M.~Wimmer, A.~Lanvin, and G.~Drettakis, ``A
  hierarchical {3D} {Gaussian} representation for real-time rendering of very
  large datasets,'' \emph{ACM Transactions on Graphics}, vol.~43, pp. 1 -- 15,
  2024.

\bibitem{fastgs}
S.~Ren, T.~Wen, Y.~Fang, and B.~Lu, ``{FastGS}: Training {3D} {Gaussian}
  splatting in 100 seconds,'' in \emph{Proceedings of the IEEE/CVF Conference
  on Computer Vision and Pattern Recognition}, 2026.

\bibitem{HansonSpeedy}
A.~Hanson, A.~Tu, G.~Lin, V.~Singla, M.~Zwicker, and T.~Goldstein,
  ``{Speedy-Splat}: Fast {3D} {Gaussian} splatting with sparse pixels and
  sparse primitives,'' in \emph{Proceedings of the IEEE/CVF Conference on
  Computer Vision and Pattern Recognition}, 2025, pp. 21\,537--21\,546.

\bibitem{hollein20243dgslm}
L.~H{\"o}llein, A.~Bo{\v{z}}i{\v{c}}, M.~Zollh{\"o}fer, and M.~Nie{\ss}ner,
  ``{3DGS-LM}: Faster {Gaussian}-splatting optimization with
  {Levenberg-Marquardt},'' in \emph{Proceedings of the IEEE/CVF International
  Conference on Computer Vision}, 2025.

\bibitem{cai2023clnerf}
Z.~Cai and M.~M{\"u}ller, ``{CLNeRF}: Continual learning meets {NeRF},'' in
  \emph{Proceedings of the IEEE/CVF International Conference on Computer
  Vision}, 2023, pp. 23\,185--23\,194.

\end{thebibliography}

\clearpage
\appendices
\setcounter{table}{0}\renewcommand{\thetable}{S-\Roman{table}}
\setcounter{figure}{0}\renewcommand{\thefigure}{S\arabic{figure}}

\section{Offline Comparison on PASLCD}
\label{app:offline}

\Cref{tab:offline} compares O\nobreakdash-SCD++ with offline SCD methods on PASLCD~\cite{galappaththige2025multi}. Offline methods receive the complete post-change capture before inferring any mask, and pairwise methods receive aligned image pairs rendered from the reference reconstruction. O\nobreakdash-SCD++ is evaluated online, emitting each mask from the frames observed up to it, and retrospectively, re-inferring every frame from its final state (\cref{sec:analysis}). Because the final state is the exact aggregate over all frames (\cref{prop:exact}), retrospective inference is the offline setting of O\nobreakdash-SCD++, directly comparable to the offline methods. Online O\nobreakdash-SCD++ exceeds every offline method, including the offline result of O\nobreakdash-SCD, and retrospective inference extends its lead.

\begin{table}[!h]
\centering
\footnotesize
\setlength{\tabcolsep}{3pt}
\caption{Offline SCD methods and O\protect\nobreakdash-SCD++ on PASLCD~\cite{galappaththige2025multi}, averaged over all 20 instances. Offline methods use the complete post-change capture. O\protect\nobreakdash-SCD++ is causal in its online setting. LF: label-free.}
\label{tab:offline}
\begin{tabular*}{\columnwidth}{@{\extracolsep{\fill}}lcccc}
\toprule
Method & LF & Setting & mIoU & F1 \\
\midrule
CYWS2D~\cite{cyws2d} & -- & offline & 0.273 & 0.398 \\
VSCD~\cite{yoon2026vscd} & -- & offline & 0.112 & 0.186 \\
GeSCF~\cite{kim2025towards} & \checkmark & offline & 0.477 & 0.611 \\
SplatPose+~\cite{liu2024splatpose+} & \checkmark & offline & 0.237 & 0.358 \\
SCAR-3D~\cite{zhou2026d} & \checkmark & offline & 0.191 & 0.289 \\
3DGS-CD~\cite{lu20253dgs} & \checkmark & offline & 0.209 & 0.339 \\
MV3DCD~\cite{galappaththige2025multi} & \checkmark & offline & 0.478 & 0.628 \\
SceneDiff~\cite{wu2025scenediff} & \checkmark & offline & 0.473 & 0.609 \\
O\nobreakdash-SCD~\cite{galappaththige2025changes} & \checkmark & offline & 0.556 & 0.700 \\
O\nobreakdash-SCD++ & \checkmark & offline (retrospective) & \textbf{0.586} & \textbf{0.726} \\
\midrule
O\nobreakdash-SCD++ & \checkmark & online & 0.575 & 0.715 \\
\bottomrule
\end{tabular*}
\end{table}

\section{Representation Update}
\label{app:update}

The conference version~\cite{galappaththige2025changes} uses its change masks to update the reference map, so that the map stays current after a revisit. O\nobreakdash-SCD++ replaces only the change detection stage, so its masks drive the same update unchanged. \Cref{tab:update} compares this update with re-optimizing the post-change scene from scratch and with continual learning, using the baseline results of~\cite{galappaththige2025changes}, and with the update driven by the refined masks of O\nobreakdash-SCD, obtained after $3000$ further iterations of offline refinement. The retrospective masks of O\nobreakdash-SCD++ require no refinement. Both updates reach higher quality than re-optimization at a fraction of its time, and are equivalent to each other: PSNR is identical, O\nobreakdash-SCD++ is marginally better in SSIM and LPIPS, and it completes faster because it skips the refinement.

\begin{table}[!h]
\centering
\footnotesize
\setlength{\tabcolsep}{2pt}
\caption{Representation update on PASLCD~\cite{galappaththige2025multi}, evaluated on novel post-change views and averaged over all 20 instances. Results of the first four methods are from~\cite{galappaththige2025changes}. O\protect\nobreakdash-SCD and O\protect\nobreakdash-SCD++ use the same update procedure, differ only in the change masks, and are timed end to end with the same script, including mask generation.}
\label{tab:update}
\begin{tabular*}{\columnwidth}{@{\extracolsep{\fill}}lcccc}
\toprule
Method & PSNR$\uparrow$ & SSIM$\uparrow$ & LPIPS$\downarrow$ & Time (s)$\downarrow$ \\
\midrule
3DGS~\cite{kerbl20233d} & 22.21 & 0.756 & 0.243 & 550 \\
3DGS-LM~\cite{hollein20243dgslm} & 22.26 & 0.756 & 0.242 & 340 \\
Speedy-Splat~\cite{HansonSpeedy} & 22.25 & 0.760 & 0.262 & 399 \\
CLNeRF~\cite{cai2023clnerf} & 22.27 & 0.624 & 0.391 & 451 \\
\midrule
O\nobreakdash-SCD~\cite{galappaththige2025changes} & \textbf{23.82} & 0.790 & 0.242 & 42 \\
O\nobreakdash-SCD++ & \textbf{23.82} & \textbf{0.792} & \textbf{0.237} & \textbf{38} \\
\bottomrule
\end{tabular*}
\end{table}

\section{Pose Estimation on Synthetic Imagery}
\label{app:changesim}

\begin{table}[!t]
\centering
\footnotesize
\setlength{\tabcolsep}{2pt}
\caption{Pose front end on the synthetic ChangeSim~\cite{park2021changesim}, averaged over the test warehouses. Reference points are triangulated from matches between reference frames, as in O\protect\nobreakdash-SCD, or lifted with depth rendered from the reference map. Posed is the share of inference images that are registered. Inference images that are not registered receive an empty mask, which scores zero IoU and lowers the overall accuracy. The XFeat+MNN rows with triangulated points are those of \cref{tab:clsplats_changesim}. MNN: mutual nearest-neighbour matching. LG: LightGlue~\cite{lightglue_2023}.}
\label{tab:changesim_pose}
\begin{tabular*}{\columnwidth}{@{\extracolsep{\fill}}lllccc}
\toprule
Method & Matcher & Ref.\ points & Posed & mIoU & F1 \\
\midrule
O\nobreakdash-SCD~\cite{galappaththige2025changes} & XFeat+MNN & triangulated & $89\%$ & 0.404 & 0.535 \\
O\nobreakdash-SCD~\cite{galappaththige2025changes} & ALIKED+LG & triangulated & $91\%$ & 0.435 & 0.572 \\
O\nobreakdash-SCD++ & XFeat+MNN & triangulated & $89\%$ & 0.474 & 0.605 \\
O\nobreakdash-SCD++ & ALIKED+LG & triangulated & $91\%$ & 0.493 & 0.626 \\
\midrule
O\nobreakdash-SCD~\cite{galappaththige2025changes} & XFeat+MNN & rendered depth & $100\%$ & 0.486 & 0.638 \\
O\nobreakdash-SCD~\cite{galappaththige2025changes} & ALIKED+LG & rendered depth & $100\%$ & 0.496 & 0.650 \\
O\nobreakdash-SCD++ & XFeat+MNN & rendered depth & $100\%$ & 0.558 & 0.706 \\
O\nobreakdash-SCD++ & ALIKED+LG & rendered depth & $100\%$ & \textbf{0.560} & \textbf{0.708} \\
\bottomrule
\end{tabular*}
\end{table}

O\nobreakdash-SCD and O\nobreakdash-SCD++ target efficient operation in real environments, so their pose front end uses XFeat~\cite{potje2024xfeat}, a lightweight descriptor designed for real images. Such lightweight descriptors, XFeat and ALIKED~\cite{zhao2023aliked} alike, match less reliably on synthetic imagery such as that of ChangeSim, as \cref{tab:changesim_pose} shows. On the real-world PASLCD and CL-Splats, the same pose front end registers every inference image, so these failures are specific to synthetic imagery and do not arise in the setting that O\nobreakdash-SCD++ targets. With reference points triangulated from XFeat matches, $11\%$ of the inference images are not registered. Each failed registration receives an empty mask that scores zero IoU, which drags down the overall accuracy. We test two changes to the pose front end that leave the change detection stage untouched. Lifting the reference keypoints with depth rendered from the reference map, rather than triangulating them from unreliable matches between synthetic reference frames, registers every inference image and raises both O\nobreakdash-SCD and O\nobreakdash-SCD++ by $0.06$ to $0.08$ mIoU, so the margin between them is preserved. Replacing the matcher with ALIKED and LightGlue~\cite{lightglue_2023} adds little once every image is registered. As in Appendix~\ref{app:keyframes}, a stronger pose front end raises the results of O\nobreakdash-SCD++ further. The pose front end is modular, so improvements in feature matching and pose estimation carry over to O\nobreakdash-SCD++ directly.

\section{Towards a Faster Pose Front End}
\label{app:keyframes}

O\nobreakdash-SCD++ removes the iterative fusion, so pose estimation, which is adopted from O\nobreakdash-SCD, now dominates the per-frame time (\cref{tab:runtime}). This appendix is a proof of concept that system-level improvements to this pose front end raise throughput further, without any change to the change detection stage. We keep it out of the main paper because it concerns the adopted front end rather than our contributions.

For every inference image, retrieval ranks all reference frames by their mutual nearest-neighbour matches before pose estimation, so its cost grows with the size of the reference capture. We restrict retrieval to a subset of reference frames, the keyframes, selected once after triangulation. We compare standard selection criteria: uniform subsampling at a fixed budget of $50\%$; voxel coverage, which adds keyframes until every voxel of the reference points, on a $12^3$ grid, is observed twice; frustum coverage, which adds keyframes until $75\%$ of the points of every reference frame lie inside the frustum of a keyframe within $30^\circ$; and sampling in pose space. The last selects keyframes by farthest-point sampling over camera position and viewing direction until every reference frame has a keyframe within $10\%$ of the scene extent and $15^\circ$. All but uniform subsampling set the number of keyframes per scene.

\begin{table}[!t]
\centering
\footnotesize
\setlength{\tabcolsep}{2pt}
\caption{Keyframe selection for retrieval on PASLCD~\cite{galappaththige2025multi}, averaged over all 20 instances. Keyframes is the share of reference frames selected as keyframes, which retrieval searches. Speed-up is the throughput gain over searching all reference frames, measured in one session. All other settings are those of O\protect\nobreakdash-SCD++.}
\label{tab:keyframes}
\begin{tabular*}{0.9\columnwidth}{@{\extracolsep{\fill}}lcccc}
\toprule
Selection & Keyframes & mIoU & F1 & Speed-up \\
\midrule
All (default) & $100\%$ & 0.575 & 0.715 & -- \\
Uniform & $50\%$ & 0.574 & 0.715 & $+9\%$ \\
Voxel coverage & $58\%$ & 0.568 & 0.708 & $+6\%$ \\
Frustum coverage & $16\%$ & 0.549 & 0.686 & $+12\%$ \\
Pose-space sampling & $43\%$ & 0.573 & 0.714 & $+9\%$ \\
\bottomrule
\end{tabular*}
\end{table}

\Cref{tab:keyframes} shows that selecting half of the reference frames as keyframes leaves accuracy essentially unchanged and raises throughput by $9\%$, and that sampling in pose space reaches the same gain with $43\%$ of the reference frames, chosen per scene. Improving pose estimation therefore offers further gains in throughput, independently of the change detection stage. We leave a dedicated study of the pose front end, including keyframe selection, cheaper reranking and lighter matching, to future work, and invite the community to build on it.

\end{document}